\documentclass[journal, twocolumn]{IEEEtran}

\ifCLASSINFOpdf
	\usepackage[pdftex]{graphicx}
	\graphicspath{{../pdf/}{../jpeg/}}
	\DeclareGraphicsExtensions{.pdf,.jpeg,.png}
\else
	\usepackage[dvips]{graphicx}
	\graphicspath{{../eps/}}
\fi

\ifCLASSOPTIONcompsoc
	\usepackage[caption=false,font=normalsize,labelfont=sf,textfont=sf]{subfig}
\else
	\usepackage[caption=false,font=footnotesize]{subfig}
\fi

\usepackage[cmex10]{amsmath}
\usepackage{bm}
\usepackage{amssymb}
\usepackage{amsthm}

\usepackage[ruled,linesnumbered]{algorithm2e}
\usepackage{algpseudocode}
\usepackage{array}
\usepackage{cite}
\usepackage{url}
\usepackage{booktabs}
\usepackage{ragged2e}
\usepackage{hyphenat}
\usepackage{color}

\renewcommand{\raggedright}{\leftskip=0pt \rightskip=0pt plus 0cm}

\renewcommand{\raggedright}{\leftskip=0pt \rightskip=0pt plus 0cm}

\newcommand{\etal}{\emph{et~al.~}}

\newtheorem{lemm}{Lemma}

\newtheorem{remark}{Remark}
\newtheorem{proposition}{Proposition}

\makeatletter

\newcommand{\Rmnum}[1]{\expandafter\@slowromancap\romannumeral #1@}
\makeatother

\begin{document}

	\title{Deep Reinforcement Learning for Traffic Light Control in Intelligent Transportation Systems}

	\author{Ming~Zhu$^+$$^*$,~\IEEEmembership{Member,~IEEE},
		Xiao-Yang~Liu$^+$,~\IEEEmembership{Member,~IEEE},
		Sem Borst,~\IEEEmembership{Member,~IEEE},
		\\and Anwar Walid,~\IEEEmembership{Fellow,~IEEE}
		\thanks{$^+$Equal contribution.}
		\thanks{$^*$Corresponding author.}

		\thanks{M.~Zhu is with Institute of Automation, Chinese Academy of Sciences, and University of Chinese Academy of Sciences,
		Beijing, China. E-mail: \nohyphens{zhumingpassional@gmail.com}.}
		\thanks{M.~Zhu was supported by National Natural Science Foundations of China (Grant No. 61902387).}
		\thanks{X.-Y.~Liu is with the Department of Electrical Engineering, Columbia University, New York, NY 10027, USA E-mail: \{xl2427\}@columbia.edu}
		\thanks{S.~Borst and A.~Walid are with Nokia Bell Labs, sem.borst@outlook.com, anwar.walid@nokia-bell-labs.com.}
	}

	\markboth{IEEE Transactions on Network Science and Engineering,~Vol.~xx, No.~xx, xxx~2025}%
	{Shell \MakeLowercase{\textit{et al.}}: Bare Demo of IEEEtran.cls for Journals}

	\maketitle

	\begin{abstract}

		Smart traffic lights in intelligent transportation systems (ITSs) are envisioned to greatly increase traffic efficiency and reduce congestion. Deep reinforcement learning (DRL) is a promising approach to adaptively control traffic lights based on the real-time traffic situation in a road network. However, conventional methods may suffer from poor scalability. In this paper, we investigate deep reinforcement learning to control traffic lights, and both theoretical analysis and numerical experiments show that the intelligent behavior ``greenwave" (i.e., a vehicle will see a progressive cascade of green lights, and not have to brake at any intersection) emerges naturally a grid road network, which is proved to be the optimal policy in an avenue with multiple cross streets. As a first step, we use two DRL algorithms for the traffic light control problems in two scenarios. In a single road intersection, we verify that the deep Q-network (DQN) algorithm delivers a thresholding policy; and in a grid road network, we adopt the deep deterministic policy gradient (DDPG) algorithm. Secondly, numerical experiments show that the DQN algorithm delivers the optimal control, and the DDPG algorithm with \textit{passive observations} has the capability to produce on its own a high-level intelligent behavior in a grid road network, namely, the ``greenwave" policy emerges. We also verify the ``greenwave" patterns in a $5 \times 10$ grid road network. Thirdly, the ``greenwave" patterns demonstrate that DRL algorithms produce favorable solutions since the ``greenwave" policy shown in experiment results is proved to be optimal in a specified traffic model (an avenue with multiple cross streets). The delivered policies both in a single road intersection and a grid road network demonstrate the scalability of DRL algorithms.

	\end{abstract}


	\IEEEpeerreviewmaketitle

	\section{Introduction}

	Smart traffic lights play an important role in intelligent transportation systems (ITSs) \cite{hong2022traffic_signal, devailly2021RL, vazifeh2018addressing,  zhu2016PublicVehicle}, since they may greatly increase traffic efficiency and reduce congestion. According to a report \cite{2021TrafficCongestionEconomy}, the annual cost caused by traffic congestion in the US has grown from \$75 billion in 2000 to \$179 billion in 2017. Vehicles' energy consumption and emission at road intersections are highly related to the control policies of traffic lights \cite{tang2017fuel_consumption}. Deep reinforcement learning (DRL) \cite{liu2020RLsurvey, liang2019DQN_traffic_light, mao2021RLTheory} is a promising approach to adaptively adjust the control policies of traffic lights according to the real-time traffic situation, e.g., the traffic throughput in each road.



	Conventional methods for traffic light control may suffer from poor scalability. 1) Classic optimization based methods \cite{koonce2008webster, roess2004greenwave, little1981maxband, cools2013SOTL, varaiya2013max_peessure, lowrie1990scats, lammer2008self, lammer2010self, albatish2019TrafficLightRule, hofri1987optimal} make decisions based on parameters calculated from traffic data, and heuristic rules obtained from expert knowledge. The number of parameters is also large if the road network is large, requiring a large amount of computing resource and leading to poor scalability. These heuristic rules are based on experiences, and may not cover all cases in complicated traffic scenarios, lacking flexibility. For example, in Maxband \cite{little1981maxband}, all intersections should share the same cycle length, which is the maximum cycle lengths of all intersections; and in Max-pressue \cite{varaiya2013max_peessure}, the pressure (the total queue length on incoming approaches minus the outgoing approaches) is only calculated when the current phase time is larger than the minimum phase duration length, which is set by expert knowledge. 2) The problem of adaptive traffic light control can be formulated as a Markov Decision Process (MDP) \cite{onori2016dynamic, puterman2014markov, ross2014introduction}, since the effects of a control action depend only on that state and not on the prior history. Conventional MDP-based methods such as value iteration, policy iteration and dynamic programming can obtain the optimal policy for relatively small-scale problems. However, the state-action space grows exponentially with the number of road intersections, becoming infeasible in a large road network. 3) Some DRL algorithms such as the deep Q-network (DQN) algorithm are not scalable either. The actions in the DQN algorithm are discrete. When the number of road intersections increases \cite{liang2019DQN_traffic_light, kumar2020traffic_light_fuzzy_RL}, the number of Q-values to be estimated increases exponentially. It is not practical to approximate such a large number of Q-values. Therefore, some DRL algorithms such as DQN may not produce a feasible solution in a large road network.

	Deep reinforcement learning algorithms \cite{haydari2020DRL_survey, mnih2015human} may not obtain a favorable solution in traffic light control. Neural networks are used to approximate the Q-function \cite{tan2019DRL_traffic_signal} and/or policy; however, they are black-box and have poor interpretation. Neural networks receive states (e.g., the number of waiting cars in each road, and the traffic light configure) and actions (e.g., ``continuing" or ``switching"), and produce the estimated Q-values without offering any interpretation. There is no guarantee for DRL to which stochastic gradient descent methods converge. In another word, whether the output action produces the maximum Q-value over all actions in a state is not guaranteed. Moreover, over-fitting or under-fitting may occur. Therefore, the estimated Q-values may not be accurate, and the output action may not produce the maximum Q-value.

	In this paper, we investigate deep reinforcement learning approaches to control traffic lights, and analyze their performance. In a single road intersection, we demonstrate that the policies obtained by the DQN algorithm matches the optimal control obtained by a conventional MDP approach. In a grid road network, we adopt the deep deterministic policy gradient (DDPG) algorithm; and it is not practical to obtain the optimal policy by conventional MDP approaches since the state-action space is too large; and we prove that the intelligent behavior ``greenwave" (i.e., a vehicle will see a progressive cascade of green lights, and does not have to brake at any intersection) is theoretically optimal in a specified traffic model. Numerical results show that the ``greenwave" policy emerges by applying the DDPG algorithm, confirming its ability to produce the high-level intelligent behavior.


	Our main contributions can be summarized as follows:
	\begin{itemize}
		\item We use two DRL algorithms for the traffic light control problems in two scenarios. In a single road intersection, we verify that the deep Q-network (DQN) algorithm delivers a thresholding policy, which matches the optimal control; and in a grid road network, we adopt the deep deterministic policy gradient (DDPG) algorithm.

		\item Numerical experiments show that the DDPG algorithm with \textit{passive observations} has the capability to produce on its own a high-level intelligent behavior in a grid road network, namely, the ``greenwave" policy emerges. We also examine the ``greenwave" patterns.

		\item The ``greenwave" patterns delivered by the DDPG algorithm demonstrate that DRL algorithms produce favorable solutions since the ``greenwave" policy is optimal in a specified traffic model (an avenue with multiple cross streets). The policies both in a single road intersection and a grid road network demonstrate the scalability of DRL algorithms in solving traffic light control problems.

	\end{itemize}

	\begin{figure*}[t]
		\centering
		\includegraphics[width=0.28\textwidth]{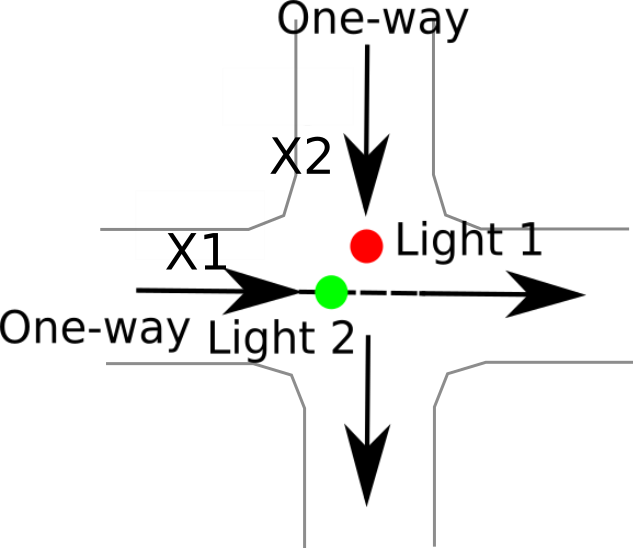}
		\hspace{0.70in}
		\includegraphics[width=0.41\textwidth]{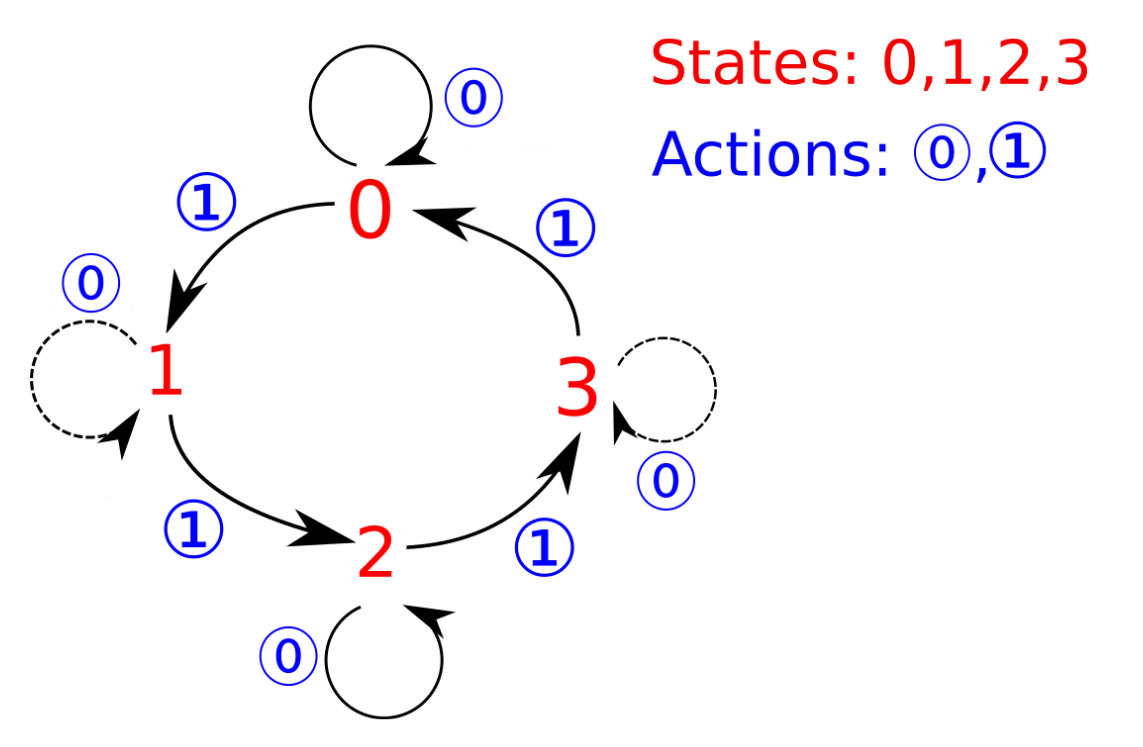}
		\caption{An intersection with two traffic flows (left), and the corresponding state-transition diagram (right).}
		\label{Fig:OneIntersectionAndStateTransition}
	\end{figure*}

	The remainder of the paper is organized as follows. Section~\ref{Sec:RelatedWorks} describes the related works. Section~\ref{Sec:Model} gives a detailed description of the road intersection model and the problem statement. Section~\ref{Sec:Solution} provides specification of DRL algorithms for a single intersection as well as a grid road network. Section~\ref{Sec:Performance} evaluates the performance of the proposed DRL algorithms and illustrates the emergence of ``greenwave" patterns. Section~\ref{Sec:Conclusion} concludes this paper.


	\section{Related Works} \label{Sec:RelatedWorks}

	Existing works on traffic light control can be classified into two categories, conventional methods and DRL methods.

	\subsection{Conventional Methods}


	Traffic light control is in fact a mixed-integer linear programming problem \cite{lin2004MILP_TrafficSignalControl} that is NP-hard. Traffic light control heavily relies on accurate traffic parameters such as the vehicles' position and speed. Conventional techniques measure them using sensors \cite{rani2017TrafficIR_IoT, kodire2016TrafficSignalGPS_ZigBee, panichpapiboon2017TrafficDensitySmartphone}, e.g., infrared sensors and inductive loop detectors. To collect the traffic data and compute policies for plenty of intersections in real-time, high data transmission rate and enough computing resource are required. Techniques such as wireless communications \cite{zhu2021RL_UAV, ma2019high} and edge cloud computing \cite{wu2020collaborate} can significantly improve the data transmission rate and computing efficiency. Conventional methods mainly include classic optimization based methods \cite{koonce2008webster, roess2004greenwave, little1981maxband, cools2013SOTL, varaiya2013max_peessure, lowrie1990scats,  lammer2008self, lammer2010self, albatish2019TrafficLightRule, hofri1987optimal}, evolutionary algorithms \cite{wang2021GA_traffic_signal, turky2009GA, brian2009stochastic}, and fuzzy logic \cite{ali2021FuzzyTrafficSignal, hawi2017fuzzy}.



	Classic optimization based methods \cite{koonce2008webster, roess2004greenwave, little1981maxband, cools2013SOTL, varaiya2013max_peessure, lowrie1990scats,  lammer2008self, lammer2010self, albatish2019TrafficLightRule, hofri1987optimal} usually adjust the cycle time and the phase splits based on parameters and heuristic rules. The time for green light is usually longer than needed, which causes a large amount of wasted time for waiting cars or pedestrians.

	\begin{itemize}
		\item Methods for single road intersection. The Webster method calculates the desired cycle length, and then calculates the green time splits, which are proportional to the ratios of critical lane volumes over each phase. Hofri \etal \cite{hofri1987optimal} proposed a thresholding policy for a single road intersection with two queues, which depends on the traffic flow data and simulation experiments; however, it may not be available in all real traffic situations.

		\item Methods for multiple road intersections. Both GreenWave \cite{roess2004greenwave} and Maxband \cite{little1981maxband} methods aim to reduce the number of stops for vehicles; however, GreenWave only optimizes only one direction, and Maxband optimizes two opposite diretions. Both GreenWave and Maxband require the same cycle length for all intersections. In the self-organizing traffic light (SOTL) \cite{cools2013SOTL} method, the requests for a green signal from the current phase and other competing phases are measured, and then certain rules are built to make decisions (``continuing" or ``switching"). Its performance is highly dependent on the rules. In the Max-pressure \cite{varaiya2013max_peessure} method, the pressure of all phases is calculated, and then choose the phase with the maximum pressure. In the Sydney coordinated adaptive traffic system (SCATS) \cite{lowrie1990scats} method, the degree of saturation (DS), i.e., the ratio of effective green time to the available green time, for each phase in the current signal plan is calculated, and then the DSs for other signal plans are calculated, and finally the plan with the minimum total DS is selected. To reduce the large fluctuations of traffic flows, Lammer \etal \cite{lammer2010self} proposed a method to adjust the duration and the order of green phases, which reaches the stabilization of queues and red-time duration.
	\end{itemize}
	These methods have poor scalability. If the road network is large, the number of parameters is also large, leading to much time spent on parameter calculating and tuning. Moreover, heuristic rules and some parameters are highly based on expert knowledge without optimality gurantee.



	Evolutionary algorithms \cite{shaikh2020EA, maier2019EA} especially genetic algorithms are usually used in multi-objective optimization, e.g., fuel consumption, number of stops, traffic flow throughput, and vehicles' waiting time \cite{wang2021GA_traffic_signal, turky2009GA, brian2009stochastic}. In theory, these algorithms are based on the concept of survival of the fittest through stochastic optimization and heuristics. Take genetic algorithms as an example. The solutions are encoded as chromosomes, and in the next generation, through reproduction, crossover, and a very small probability of mutation, new offsprings are generated, and the ones with large fit values will survive. In this way, in each generation or iteration, the solutions move nearer to the local or even global optimal points. Brian \etal \cite{brian2009stochastic} proposed a genetic algorithm to minimize fuel consumption, emission, delay and number of stops. Turky \etal \cite{turky2009GA} used a genetic algorithm to dynamically control the red and green duration to optimize the vehicle and the pedestrian flow throughput. However, the performance of genetic algorithms may depend on some parameters such as the population size, the number of generations, and the mutation probability, and there is no performance guarantee. For example, genetic algorithms may not obtain successful solutions if the population size is small. Therefore, evolutionary algorithms may obtain an unfavorable solution within limited time.


	Fuzzy logic \cite{ali2021FuzzyTrafficSignal, hawi2017fuzzy} can be used in traffic light control problems, since the ambiguity in the linguistic terms (e.g., low-, medium- and high-speed) can be represented using fuzzy sets. To compensate the traffic flow fluctuation, Ali \etal \cite{ali2021FuzzyTrafficSignal} used fuzzy logic to monitor and handle the alternation of the traffic condition between two successive cycles. Hawi \etal \cite{hawi2017fuzzy} used fuzzy logic and wireless sensor networks to control traffic lights, where wireless sensor networks collect traffic data in real-time and fuzzy logic outputs the order of the green light assignment with traffic quantity and waiting time as the input. In the above methods, designing the fuzzy logic needs high human expertise and regular updating of rules, and generally accurate reasoning cannot be given. Moreover, the performance depends on the heuristic rules and parameters, and there is no guarantee of stability or optimality.


	\subsection{Deep Reinforcement Learning Methods}

	DRL methods have various formulations, including state, action, and reward, leading to various results. Based on formulations, DRL algorithms are applied, and these algorithms are classified into single-agent and multi-agent ones.

	First, we discuss DRL formulations.
	\begin{itemize}
		\item State formulation. The state is generally the combinations of the following items: queue length \cite{wei2018intellilightRL}, waiting time \cite{chu2019MARL}, volume \cite{balaji2010urban} (number of vehicles on the lane/road), delay \cite{arel2010MARL} (realistic travel time minus the expected travel time), speed, phase duration (of the current phase) \cite{mannion2016RL}, congestion \cite{bakker2010MARL} (denoted by 0/1 or the congestion level), and image or matrix \cite{mousavi2017RL, wei2018intellilight} (demonstrating positions of vehicles, where ``0" denotes absence and ``1" denotes presence), etc. There is a trend of using image/matrix as the state recent years; however, the dimension of state is very large, e.g., more than one thousand, which leads to long training time.

		\item Action formulation. The action generally has four types: set current phase duration \cite{aslani2017adaptive}, set phase ratio \cite{abdoos2011MAQLearning}, continuing/switching \cite{van2016RL}, and select the next phase \cite{zang2020metalight}. The former two types of actions are continuous, and the latter two ones are discrete. Selecting the next phase is more flexible since the traffic light does not need to change cyclically. Here, we select continuing/switching, since it is comprehensive.

		\item Reward formulation. The action generally includes the following items: queue length \cite{wei2018intellilightRL}, waiting time \cite{van2016RL}, throughput \cite{aslani2017adaptive}, speed \cite{van2016RL}, and pressure \cite{chen2020DQN_traffic_light} (the number of incoming vehicles minus the outgoing ones). Some research works use a weighted linear combination of the above items as reward. The tuning of weights is an important issue, and minor changes for weights may lead to very different results. To avoid this case, we only select queue length as the reward.
	\end{itemize}

	Secondly, we discribe single-agent and multi-agent DRL approaches. In single-agent DRL approaches, the agent learns the Q-value for all actions of each state, and chooses the action with maximum approximated Q-value on that state. Through a bootstrap learning process, the agent obtains a policy for an end-to-end task with raw input data. These methods are only applied in small or medium-sized road networks due to poor scalability. DQN and its variants \cite{mnih2015human, liang2019DQN_traffic_light} are commonly used methods. Liang \etal \cite{liang2019DQN_traffic_light} proposed to use the dueling double DQN with prioritized experience replay to decide the traffic signal duration. The single-agent DRL methods have a global view of the traffic situations; however, they may not be practical in very large-scale road networks.

	In multi-agent DRL approaches \cite{Wu2020MARDDPG, chu2019MARL, xu2019TrafficSignalMARL, liu2014TrafficSignalMARL, liu2017TrafficSignalMARL, zhang2021MARL}, each agent controls a traffic light in an intersection, and also considers the states and actions of its neighborhoods to cooperate with them. The sum of weighted reward of neighbor agents plus its own is then treated as a global reward. Liu \etal \cite{liu2014TrafficSignalMARL} proposed a cooperative Q-learning with a function approximation algorithm so that each local intersection can cooperatively make decisions without any central supervising agents. Chu \etal \cite{chu2019MARL} proposed an actor-critic based multi-agent DRL method to stabilize the learning procedure, improve the observability and reduce the learning difficulty. Xu \etal \cite{xu2019TrafficSignalMARL} proposed a multi-agent DRL approach to optimize the signal timing plans at intersections. Wu \etal \cite{Wu2020MARDDPG} proposed a multi-agent DDPG method for traffic light control, where long short-term memory (LSTM) \cite{yu2019LSTM} is used to improve the stability of the environment caused by partial observable state. Multi-agent DRL methods can control traffic lights in a cooperative way among agents in a very large-scale road network. However, these methods generally may not have a global view of the traffic situations needed for the favorable control policy. Moreover, how to balance the utilities of an agent and its neighborhoods is also a hard problem.



	The above DRL algorithms use the black-box neural networks to fit the Q-function, and stochastic gradient descent methods to minimize the non-convex loss function. These methods cannot provide guarantee the scalablity. In DQN methods, the number of approximated Q-values under each state over all actions increases exponentially with the increasing of the number of road intersections. The DQN algorithm needs too much computing resource and time to estimate the Q-values in a large road network, leading to the poor scalability. Multi-agent DRL methods such as multi-agent DDPG \cite{Wu2020MARDDPG} also may not guarantee for the scalability, since they are black-box and the optimal policy may not be obtained.



	Different from the existing methods, we apply DRL methods, namely, DQN and DDPG, to optimize the control policy for traffic lights in a single road intersection and a grid road network, respectively. Considering that the DQN algorithm has poor scalability, it is only used in a single road intersection. Considering that the DDPG algorithm can handle complicated scenarios using continuous actions, it is used in a grid road network. We obtain the optimal policy in a single road intersection using a conventional method, and obtains the optimal policy in a grid road network under specified traffic model, i.e., ``greenwave" policy. Simulation results show that the DQN algorithm delivers the optimal control, and the proposed DDPG algorithm produces a high-level intelligent behavior in real-time, i.e., the ``greenwave" pattern emerges.

	\section{Model Description and Problem Formulation} \label{Sec:Model}

	In this section, we present the models and optimization goals for single and multiple traffic lights. For the ease of exposition, we consider an admittedly stylized model which captures the most essential features governing the dynamics of traffic flows at road intersections. Our central aim is to broadly explore the performance and scalability of DRL algorithms in optimizing real-time control policies, rather than to develop a practical policy for a specific instance. Thus, we adopt a discrete-time formulation to simplify the description and allow conventional MDP techniques for comparison. The methods and results can naturally extend to continuous-time operation.

	\subsection{Model for Single Traffic Light} \label{Subsec:SingleRoadIntersection}

	\begin{figure*}[t]
		\centering
		\includegraphics[width=0.71\textwidth]{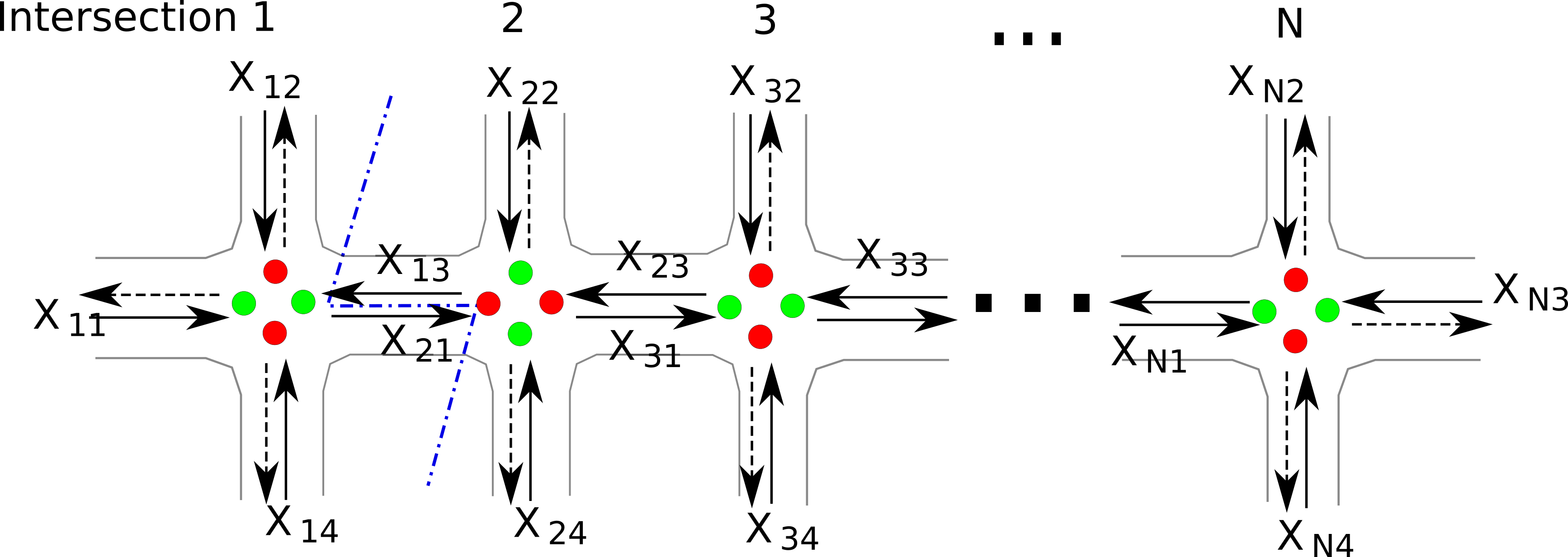}
		\caption{Illustration of a road network with grid topology.}
		\label{Fig:LinearRoadNetwork}
	\end{figure*}

	We start with a single intersection scenario to facilitate the validation of DRL algorithms. We consider the simplest meaningful setup with two intersecting unidirectional traffic flows as schematically depicted in the left part of  Fig.~\ref{Fig:OneIntersectionAndStateTransition}. We assume two directions with 1 and 2 denoting a west-to-east direction (avenue) and a north-to-south direction (cross street), respectively. Therefore, there are two traffic flows denoted by 1 and 2. The state $S(t)$ at the beginning of time slot~$t$ may be described by the three-tuple $(X_1(t), X_2(t); L(t))$, with $X_i(t), ~i \in\{ 1, 2 \}$ denoting the number of vehicles of traffic flow~$i$ waiting to pass the intersection and $L(t) \in \{0, 1, 2, 3\}$ denoting the configuration of the traffic lights:
	\begin{itemize}
		\item ``0", i.e., `green': green light for flow~$1$ (avenue) and hence red light for flow~$2$ (cross street);
		\item ``1", i.e., `yellow': yellow light for flow~$1$ (avenue) and hence red light for flow~$2$ (cross street);
		\item ``2", i.e., `red': red light for flow~$1$ (avenue) and hence green light for flow~$2$ (cross street);
		\item ``3", i.e., `orange': red light for flow~$1$ (avenue) and hence yellow light for flow~$2$ (cross street).
	\end{itemize}

	Each configuration~$L(t)$ can either be simply continued in the next time slot or otherwise be switched to the natural subsequent configuration $(L(t) + 1) \mod 4$. This is determined by the action $A(t) \in \{0, 1\}$ selected at the end of time slot~$t$, which is represented by a binary variable: ``0" for continuing, and ``1" for switching. Then we have
	\begin{equation} \label{Eqn:LightChangeOneIntersection}
	L(t + 1) = (L(t) + A(t)) \mod 4.
	\end{equation}
	These rules give rise to a strictly cyclic control sequence as illustrated in the right part of Fig.~\ref{Fig:OneIntersectionAndStateTransition}.

	The evolution of the queue state over time is governed by the recursion
	\begin{equation}
		\begin{split}
			& (X_1(t+1), X_2(t+1)) \\
			&= (X_1(t) + C_1(t) - D_1(t), X_2(t) + C_2(t) - D_2(t)),
			\label{Eqn:QueueStateOneIntersection}
		\end{split}
	\end{equation}
	with $C_i(t)$ denoting the number of vehicles of traffic flow~$i$ appearing at the intersection during time slot~$t$ and $D_i(t)$ denoting the number of departing vehicles of traffic flow~$i$ crossing the intersection during time slot~$t$. We make the simplified assumption that if a traffic flow is granted the green light, exactly one waiting vehicle, if any, will cross the intersection during that time slot, i.e.,
	\begin{equation}
		\begin{split}
			D_1(t) &= \min\{1, X_1(t)\},~~\text{if}~ L(t) = 0,\\
			D_1(t) &= 0, ~~~~~~~~~~~~~~~~~~\text{if}~  L(t) \neq 0; \\
			D_2(t) &= \min\{1, X_2(t)\},~~\text{if}~ L(t) = 2, \\
			D_2(t) &= 0, ~~~~~~~~~~~~~~~~~~\text{if}~  L(t) \neq 2.
			\label{Eqn:DepartingFlowOneIntersection}
		\end{split}
	\end{equation}

	\subsection{Model for Multiple Traffic Lights}
	\label{Subsec:LinearGridRoad}

	To examine the scalability property of DRL algorithms in more complex scenarios, we consider a grid road network as depicted in Fig.~\ref{Fig:LinearRoadNetwork}. Specifically, we investigate a grid road network with $N$ intersections and bidirectional traffic flows, representing an avenue with multiple cross streets. We do not account for any traffic flows making left or right turns, since the model could be easily generalized to accommodate it. The state $S(t)$ at the beginning of time slot~$t$ may be described by a $(5 \, N)$-tuple $(X_{n1}(t), X_{n2}(t), X_{n3}(t), X_{n4}(t); L_n(t))_{n \in \{ 1, ..., N \}}$. The first four components denote the number of vehicles waiting to cross the $n$-th intersection with 1, 2, 3, and 4 denoting the west-to-east, north-to-south, east-to-west, and south-to-north directions, respectively. The last component $L_n(t) \in \{ 0, 1, 2, 3 \}$ indicates the configuration of the traffic light at the $n$-th intersection similar as before:
	\begin{equation}
		L_n(t+1) = (L_n(t) + A_n(t)) \mod 4, \label{Eqn:LightChangeMultipleIntersection}
	\end{equation}
	with $A_n(t) \in \{0, 1\}$ denoting the action selected for the $n$-th road intersection at the end of time slot~$t$. The evolution of the various queue states is determined by the recursion
	\begin{equation} \label{Eqn:QueueStateMultipleIntersection}
	X_{ni}(t+1) = X_{ni}(t) + C_{ni}(t) - D_{ni}(t),
	\end{equation}
	with $C_{ni}(t)$ denoting the number of vehicles in direction~$i$ appearing at the $n$-th intersection during time slot~$t$ and $D_{ni}(t)$ denoting the number of vehicles in direction~$i$ crossing the $n$-th intersection during time slot~$t$, $i \in \{ 1, ..., 4 \}$, $n \in \{ 1, ..., N \}$.
	Note that
	\begin{equation} \label{Eqn:DepartingFlowMultipleIntersection}
	\begin{split}
		&D_{n1}(t) = \min\{1, X_{n1}(t)\}, ~~D_{n3}(t) = \min\{1, X_{n3}(t)\},\\
		&~~~~~~~~~~~~~~~~~~~~~~~~~~~~~~~~~\text{if}~L_n(t) = 0,\\
		&D_{n1}(t)=0,~~D_{n3}(t) = 0, ~\, \text{if}~L_n(t) \neq 0;\\
		&D_{n2}(t) = \min\{1, X_{n2}(t)\}, ~~D_{n4}(t) = \min\{1, X_{n4}(t)\},\\
		&~~~~~~~~~~~~~~~~~~~~~~~~~~~~~~~~~\text{if}~L_n(t) = 2,\\
		&D_{n2}(t)=0,~~D_{n4}(t) = 0, ~\, \text{if}~L_n(t) \neq 2.
	\end{split}
	\end{equation}

	\vspace{0.05in}
	\vspace{0.05in}
	\vspace{0.05in}
	\begin{table*}[tbp]
		\normalsize
		\begin{tabular}{lp{0.45\textwidth}}
			\toprule
			\textbf{Algorithm 1}: Environment simulation (one step) for a single road intersection or a grid road network\\
			\toprule
			\,\,\,1:~Select action $a$ according to line 3 in Alg.~2 [a single road intersection] or line 3 in Alg.~3 [a grid road network]; \\
			\,\,\,2:~Generate random variables $C_1$, $C_2$ [a single road intersection] or $C_{n1}$, $C_{n3}$ and $C_{n2}$, $C_{n4}$ for all $n \in \{ 1, ..., N \}$ \\
			\,\,\,~~~[a grid road network]; \\
			\,\,\,3:~Execute action $a$, determine the new queue states and traffic light states $X_1', X_2', L'$ according to \eqref{Eqn:LightChangeOneIntersection} $\sim$ \eqref{Eqn:DepartingFlowOneIntersection} \\
			\,\,\,~~~~[a single road intersection] or $X'_{n1}, X'_{n2}$, $X'_{n3}, X'_{n4} = 0$; $L'_n = 0$ for all $n \in \{ 1, ..., N \}$ according to \eqref{Eqn:LightChangeMultipleIntersection} $\sim$ \eqref{Eqn:DepartingFlowMultipleIntersection} \\
			\,\,\,~~~~[a grid road network];\\
			\,\,\,4:~Calculate the reward $r = - (|X_1'|^2 + |X_2'|^2)$ [a single road intersection] or $r = - \sum_{n=1}^{N} \sum_{i = 1}^{4} |X_{ni}'|^2$ \\
			\,\,\,~~~[a grid road network]; \\
			\,\,\,5:~Update states: $X_1 = X_1', X_2 = X_2', L = L'$ [a single road intersection] or $X_{n1} = X'_{n1}, X_{n2} = X'_{n2}$, $X_{n3} = X'_{n3},$ \\
			\,\,\,~~~$ X_{n4} = X'_{n4}$, $L_n = L'_n$ for all $n \in \{ 1, ..., N \}$ [a grid road network].\\
			\bottomrule
		\end{tabular}
	\end{table*}
	\vspace{0.05in}
	\vspace{0.05in}
	\vspace{0.05in}

	\vspace{0.05in}
	\vspace{0.05in}
	\vspace{0.05in}
	\begin{table*}[tbp]
		\normalsize
		\begin{tabular}{lp{0.45\textwidth}}
			\toprule
			\textbf{Algorithm 2}: DQN for a single road intersection\\
			\toprule
			\,\,\,1:~Initialize $X_1 = X_2 = 0, L = 0$, action-value $Q$ with random weights $\theta_0$, batch size $M$, number of time steps $K$, \\
			\,\,\,~~~and replay buffer $D$;\\
			\,\,\,2:~\textbf{for} steps $k = 1, \dots, K$ \\
			\,\,\,3:~~~Select $a = \arg\underset{\bar{a} \in \mathcal{A}}\max~ Q(s, \bar{a})$ with probability $1 - \epsilon$ or a random action $a$ with probability $\epsilon$;\\
			\,\,\,4:~~~Execute action $a$, get new state $s' = [X_1', X_2'; L']$ and reward $r$ from Alg.~1;\\
			\,\,\,5:~~~Store transition $(s, a, r, s')$ in replay buffer $D$; \\
			\,\,\,6:~~~Sample a mini-batch of $M$ transitions $\{ (s^j_t, a^j_t, r^j_t, s^j_{t + 1}) \}_j$;\\
			\,\,\,7:~~~Update the networks using Adam optimizer: \\
			~~~~~~~~~$\nabla_{\theta_k} \text{Loss}(\theta_k) = \mathbb{E}_{s^j_t, a^j_t, r^j_t, s^j_{t + 1}}\bigg[(r+\gamma\max_{a^j_{t + 1}}\big(Q(s^j_{t + 1},a^j_{t + 1}|\theta_{k - 1})\big) -Q(s^j_t,a^j_t|\theta_k)) \nabla_{\theta_k} Q(s^j_t,a^j_t|\theta_k)\bigg]$.\\
			\bottomrule
		\end{tabular}
	\end{table*}
	\vspace{0.05in}
	\vspace{0.05in}
	\vspace{0.05in}

	\subsection{Objective Function for Optimization}
	\label{Subsec:OptimizationGoal}

	We assume that the ``congestion cost" in time slot~$t$ may be expressed as a function of the queue state. A quadratic function $Z(X(t))$ is applied to formulate the congestion cost. In the single road intersection scenario, $X(t) = (X_1(t), X_2(t))$, and $Z(X(t)) = X_1^2(t) + X_2^2(t)$. In the grid road network with $N$~intersections, $X(t) = (X_{n1}(t), X_{n2}(t), X_{n3}(t), X_{n4}(t))_{n \in \{ 1, ..., N \}}$, and $Z(X(t)) = \sum_{n = 1}^{N} \sum_{i = 1}^{4} X_{ni}^2(t)$.

	The optimization goal is to find a dynamic control policy which selects actions over time slots $\{ 0,1,...,T \}$ to minimize the long-term expected discounted cost $\mathbb{E}\left[\sum_{t=0}^{T} \gamma^t Z(X(t))\right]$, with $\gamma \in (0, 1]$ representing a discount factor.

	\section{Deep Reinforcement Learning Algorithms} \label{Sec:Solution}

	\vspace{0.05in}
	\vspace{0.05in}
	\vspace{0.05in}
	\begin{table*}[tbp]
		\normalsize
		\begin{tabular}{lp{0.45\textwidth}}
			\toprule
			\textbf{Algorithm 3}: DDPG for a grid road network\\
			\toprule
			\,\,\,1:~Initialize batch size $M$, number of time steps $K$, and replay buffer $D$, the actor and the critic's online networks  \\
			\,\,\,~~~($\mu(s|\theta^\mu)$ and $Q(s,a|{\theta^Q})$) with parameters $\theta^\mu$ and $\theta^Q$, and target networks ($\mu'(s|\theta^{\mu'})$ and $Q'(s,a|\theta^{Q'})$) with parameters; \\
			\,\,\,~~~$\theta^{\mu'}$ and $\theta^{Q'}$, queue states and traffic light states: $X_{n1}, X_{n2}$, $X_{n3}, X_{n4} = 0$; $L_n = 0$ for all $n \in \{ 1, ..., N \}$; \\
			\,\,\,2:~\textbf{For} steps $k = 1, \dots, K$  \\
			\,\,\,3:~~~~~Choose action $a = \mu(s|\theta^{\mu}) + \mathcal{N}$, where $\mathcal{N}$ is Gaussian noise; \\
			\,\,\,4:~~~~~Execute action $a$, get new state $s' = [X_{n1}, X_{n2}$, $X_{n3}, X_{n4}; L_n]$ for all $n \in \{ 1, ..., N \}$ and reward $r$ from Alg.~1;\\
			\,\,\,5:~~~~~Store transition $(s, a, r, s')$ in memory; \\
			\,\,\,6:~~~~~Sample a mini-batch of $M$ transitions $\{ (s^j_t, a^j_t, r^j_t, s^j_{t + 1}) \}_j$; \\
			\,\,\,7:~~~~~Calculate the target Q-value: $y^j = r^j_t + \gamma Q'(s^j_{t + 1}, \mu'(s^j_{t + 1}| \theta^{\mu'})|\theta^{Q'})$; \\
			\,\,\,8:~~~~~Update the critic's online Q-network $Q$ by minimizing the loss function: \\
			~~~~~~~~\,$\nabla_{\theta^Q} \text{Loss}_t(\theta^Q)= \nabla_{\theta^Q} [\frac{1}{M} \sum^{M}_{j = 1} (y^j - Q(s^j_t, a^j_t| \theta^Q))^2]$; \\
			\,\,\,9:~~~~~Update the actor's online policy network with the sampled policy gradient by the chain rule: \\
			~~~~~~~~~$\mathbb{E}_{s_t} [\nabla_{\theta^{\mu}} Q(s, a| \theta^Q)|_{s = s_t, a = \mu(s_t|\theta^{\mu})}] = \mathbb{E}_{s_t} [\nabla_{a} Q(s, a| \theta^Q)|_{s = s_t, a = \mu(s_t)} \nabla_{\theta^{\mu}} \mu(s|\theta^{\mu})|_{s = s_t} ] $; \\
			10:~~~~~Update the critic's target Q-network $Q'$ and actor's target policy network $\mu'$: \\
			~~~~~~~~~$\theta^{Q'}  \leftarrow \tau \theta^Q + (1 - \tau) \theta^{Q'}$, $\theta^{\mu'}  \leftarrow \tau \theta^{\mu} + (1 - \tau) \theta^{\mu'}$ with $\tau \ll 1$. \\
			\bottomrule
		\end{tabular}
	\end{table*}
	\vspace{0.05in}
	\vspace{0.05in}
	\vspace{0.05in}


	In this section, we first show a brief description of the reinforcement learning and Q-Learning. Then we propose two deep reinforcement learning schemes, DQN \cite{mnih2015human} and DDPG \cite{lillicrap2016DDPG} algorithms in different scenarios. We use the DQN algorithm in a single road intersection in Fig.~\ref{Fig:OneIntersectionAndStateTransition}. While we use the DDPG algorithm in a grid road network in Fig.~\ref{Fig:LinearRoadNetwork}. The DDPG algorithm applies the actor-critic framework in continuous action space. The traffic lights' actions in our experiments are discrete, i.e., $A(t) \in \{0, 1\}$, thus we apply a discretization process to the conventional DDPG algorithm.
	Finally, we demonstrate that the ''greenwave" control policy is optimal in a grid road network.

	\subsection{Reinforcement Learning and Q-Learning}

	\textbf{Definition of cost function}. We define a \textit{discounted congestion cost} over time slots $\{ 0,1,..., T \}$ of a road intersection under policy $\pi$ with the discount factor $\gamma$
	\begin{equation}
		V_\pi= - \frac{1}{T}\mathbb{E} \left[ \sum_{t=0}^{T}\gamma^t|X(t)|^2 \right],
	\end{equation}
	where $|X(t)|^2\overset{\text{def}}{=}|X_1(t)|^2+|X_2(t)|^2$ in a single intersection or $|X(t)|^2\overset{\text{def}}{=} \sum_{n=1}^{N} \sum_{i = 1}^{4} |X_{ni}(t)|^2$ in a grid road network. The optimization goal in Subsection \ref{Subsec:OptimizationGoal} is to minimize the long-term expected queue length, thus we take the negative value to be consistent with the maximization operation. The optimal policy $\widehat{\pi}$ satisfies
	\begin{equation}
		V_{\widehat{\pi}}= - \sup \frac{1}{T}\mathbb{E} \left[ \sum_{t=0}^T\gamma^t|X(t)|^2 \right].
	\end{equation}
	Combining the Bellman optimality equation \cite{sutton2018RL} and $V_{\widehat{\pi}}=-\max\mathbb{E} [ \sum_{t=0}^T\gamma^t|X(t)|^2 ]$, we derive the following equation
	\begin{equation}
		V_{\widehat{\pi}}=\max\{r_t + \gamma \max\sum_{k=1}^{T-t}\gamma^{k-1} r_{t+k}\},
	\end{equation}
	where $r_t=-|X(t)|^2$. We define $Q(s, a)$ to be the maximum achievable expected discounted reward (or minimum discounted congestion cost in our context) under the optimal policy starting from state $s = (X; L)$ when action~$a$ is taken. $Q(s, a)$ satisfies the equation
	\begin{equation}
		\begin{split}
			Q(s, a)
			&= r(s, a) + \gamma \sum\limits_{s' \in \mathcal{S}} p(s, s'; a)~\max\limits_{a' \in \mathcal{A}}~ Q(s', a')  \\
			&= r(s, a) + \gamma \mathbb{E}\left[\underset{a' \in \mathcal{A}}\max~ Q(s', a')\right],
			\label{qvalues}
		\end{split}
	\end{equation}
	with $r(s, a) = -Z(X)$ denoting the reward (i.e., negative congestion cost) in queue state~$X$, and $p(s, s'; a)$ denoting the transition probability from state~$s$ to state~$s'$ when action~$a$ is taken.

	\subsection{DQN for Single Road Intersection}

	In the DQN algorithm \cite{mnih2015DQN}, the neural network takes the state and action as the input, and the action-value as the output. Compared with traditional Q-learning algorithms, the neural network of DQN is used as a function approximator to estimate the action-value function. The neural network is trained by adjusting its parameters in each iteration to reduce the mean-squared error in the Bellman equation.

	In addition, the DQN algorithm uses an experience replay buffer to update network parameters. It has two advantages: 1) enabling the stochastic gradient decent algorithm \cite{daniely2017SGD}; and 2) removing the correlations between consecutive transitions.


	We simulate the environment using use Alg.~1. Line 1 selects an action according to DRL algorithms. Line 2 generates the number of arrival vehicles of each traffic flow. Line 3 calculates new queue states and traffic light states. Line 4 calculates the reward. Finally, line 5 updates all states.

	Alg.~2 shows the details of the DQN for a single road intersection. Line 1 is the initialization. Lines 2 $\sim$ 7 are a $K$-step loop of the train process. Line 3 determines the action. Line 4 executes the action, and updates the new queue states and traffic light states, and get the reward. Line 5 stores the state transition. Line 6 samples a mini-batch of transitions, and line 7 updates the networks.

	\subsection{Scalable DDPG for Grid Road Network}

	The DDPG algorithm \cite{lillicrap2016DDPG} applies neural networks, i.e., a critic and an actor, to approximate the action-value and action policy, respectively. The critic takes the state and action as the input, and action-value as the output. The actor takes the state as the input, and the action policy as the output. This actor-critic approach has two advantages: 1) it uses neural networks as approximators, essentially compressing the state and action space to much smaller latent parameter space, and 2) the gradient descent method can be used to update the network weights, which greatly speeds up the convergence and reduces the training time. Therefore, the memory and computational resources are largely saved. The DDPG algorithm has shown impressive and powerful skills in AlphaGo zero \cite{silver2017Game} and Atari game playing \cite{mnih2013AtariDRL}.

	Besides the actor-critic approach, the DDPG algorithm mainly uses three techniques to solve MDP problems: experience replay buffer, soft update, and exploration noise. The experience replay buffer stores transitions just as described in the DQN algorithms. Soft update with a low learning rate is introduced to improve the stability of learning. Exploration noise is added to the actor's target policy to obtain a new exploration policy, which helps to explore the other space to get better policy.

	The DDPG algorithm is first trained offline using memory as the sample pool. In the training process, the DDPG algorithm samples transitions online, and trains the agent on collected data offline in memory for further learning. The networks are updated with the Adam optimizer \cite{kingma2014adam} gradient-descent. In the test process, the DDPG algorithms run in an online form without the exploration noise.


	The conventional DDPG algorithm has a continuous action space, while the traffic control actions in our model are discrete, i.e., $A_n(t) \in \{0, 1\}$. Therefore, we apply a discretization process to transform the continuous outputs of the actor network to discrete ones. A modified sigmoid function is added to the output layer of actor network as activation:
	\begin{equation}
		\widehat{y}=\text{sigmoid}(\widehat{\alpha} \,  \widehat{x}),
	\end{equation}
	where $\widehat{x}, \widehat{y}$ are inputs and outputs of the final layer, and $\widehat{\alpha}$ is the ratio for steepening the sigmoid function. Combining the modified sigmoid activation function with a node-wise binarization process, our discrete DDPG algorithm can reduce the errors caused by the continuous-to-discrete transformation to a great degree.

	To converge to a favorable strategy in a grid road network, we set the number of cars and traffic light states at all intersections as inputs in DDPG. Let's take intersection $n$ as an example. The number of cars on each queue is $[X_{n1}, X_{n2}, X_{n3}, X_{n4}]$, and a light state is $L_n$. The observation tuple $S = [X_{n1}, X_{n2}, X_{n3}, X_{n4}, L_n]$ is a quintuple, which is also the inputs of both the critic's online Q-network and the actor's online policy network.


	Alg.~3 shows the details of DDPG for a grid road network. Line 1 is the initialization. In the following for-loop (line 2 $\sim$ 6), we train the policy. Line 3 chooses an action based on the actor's online policy network and the exploration noise. Line 4 executes the action and gets the new state and reward. Line 5 stores the state transition. Line 6 samples a mini-batch of transitions. Line 7 calculate the target Q-value, and line 8 updates the critic's online Q-network. Line 9 updates the actor's online policy network. Line 10 updates the critic's target Q-network $Q'$ and actor's target policy network $\mu'$ based on their online networks, respectively.

	\subsection{Optimality of the ``Greenwave" Policy} \label{Subsec:Optimality}

	In this subsection, we demonstrate the optimality of the ``greenwave" control policy in a grid road network with symmetric fluid traffic.

	To obtain the optimal policy in a grid road network, we first obtain the optimal policy in a single intersection, and then extend it to a grid road network to obtain the ``greenwave" policy.

	In a single intersection scenario, the optimal policy is a ``thresholding policy" \cite{hofri1987optimal}. The traffic light changes only when the difference between two queue lengths reaches a critical value.

	In the grid road network scenario, we assume there is an avenue (i.e., an artery road) with multiple cross streets. We imagine that all the intersections constitute a large ``intersection". The ``greenwave" policy is defined as follows: the traffic lights simultaneously change only when the difference between the avenue's queue length and the cross streets' queue length reaches a critical value. For a vehicle in the avenue, it will see a progressive cascade of green lights, and not have to brake at any intersection. Therefore, this control policy is called ``greenwave". The ``greenwave" policy achieves the optimality in a grid road network (an avenue with multiple cross streets). The detailed proof is moved to the appendix.

	\section{Performance Evaluation} \label{Sec:Performance}

	\begin{figure*}[htbp]
		\includegraphics[width=0.33\textwidth]{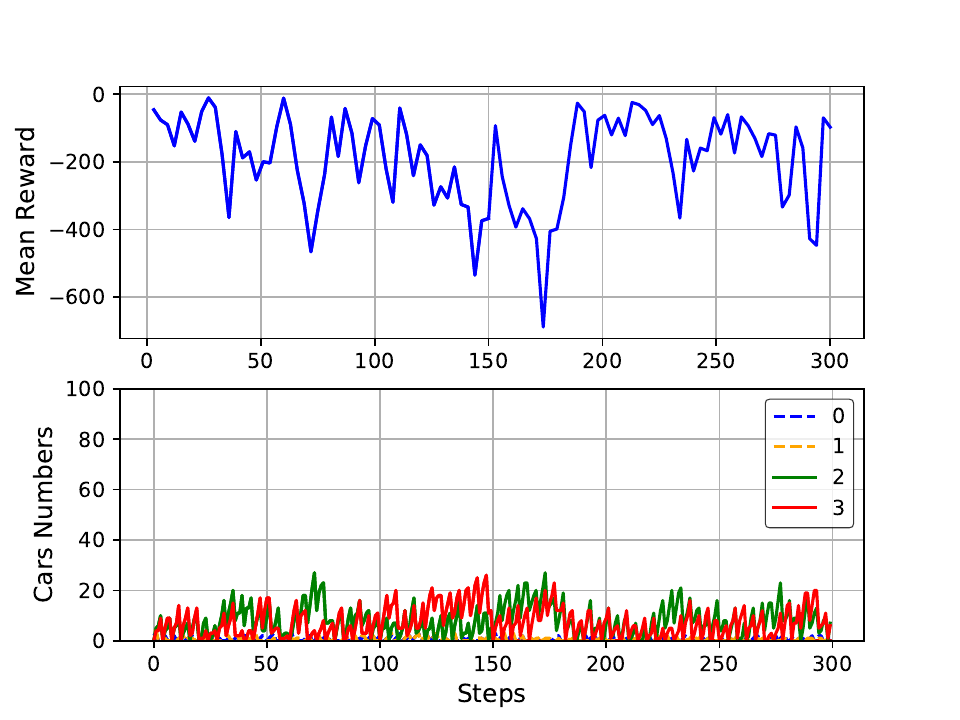}
		\includegraphics[height=1.79in,width=0.33\textwidth]{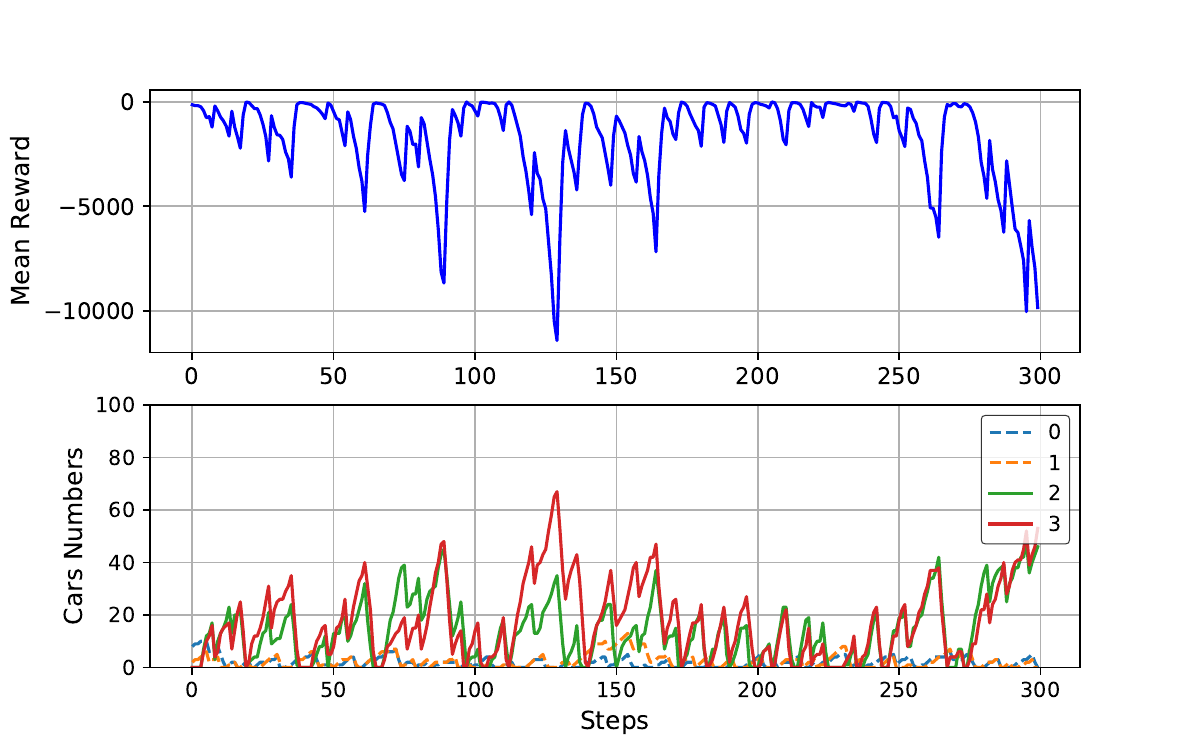}
		\includegraphics[width=0.33\textwidth]{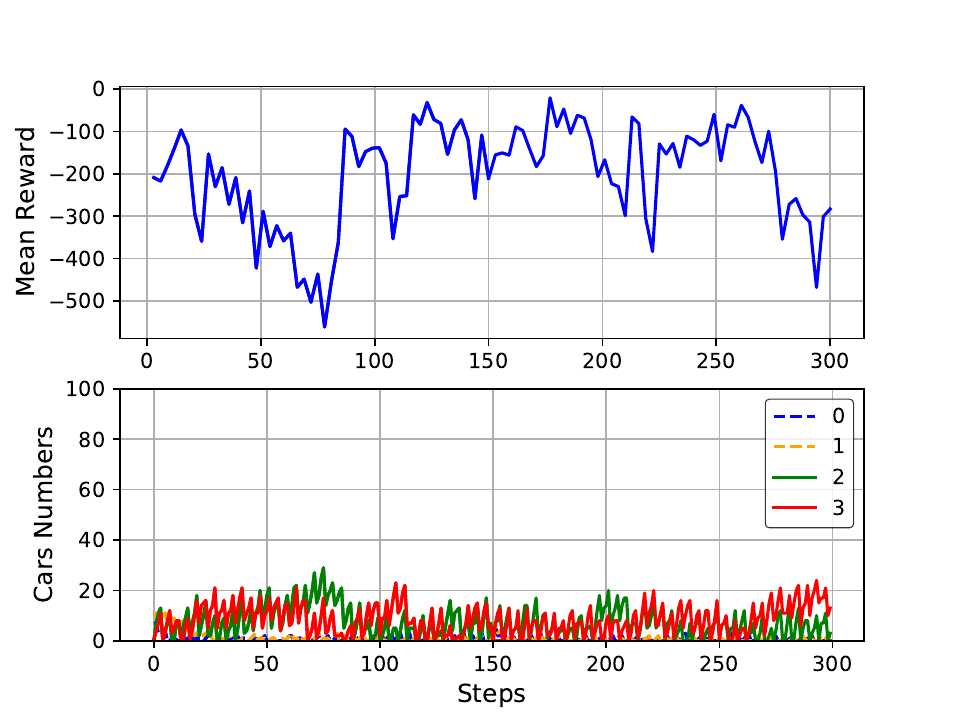}
		\caption{Performance of DQN control policy (left), fixed-cycle policy (middle), and the optimal policy (right) for single intersection. Tests of 300 steps in simulation are shown with average reward and number of waiting cars. Better policies are supposed to reduce the numbers of waiting cars. The DQN algorithm matches the optimal policy.}
		\label{fig:SingleIntersection}
	\end{figure*}

	\begin{figure*}
		\centering
		\begin{minipage}{0.46\textwidth}
			\centering
			\includegraphics[width=0.95\linewidth]{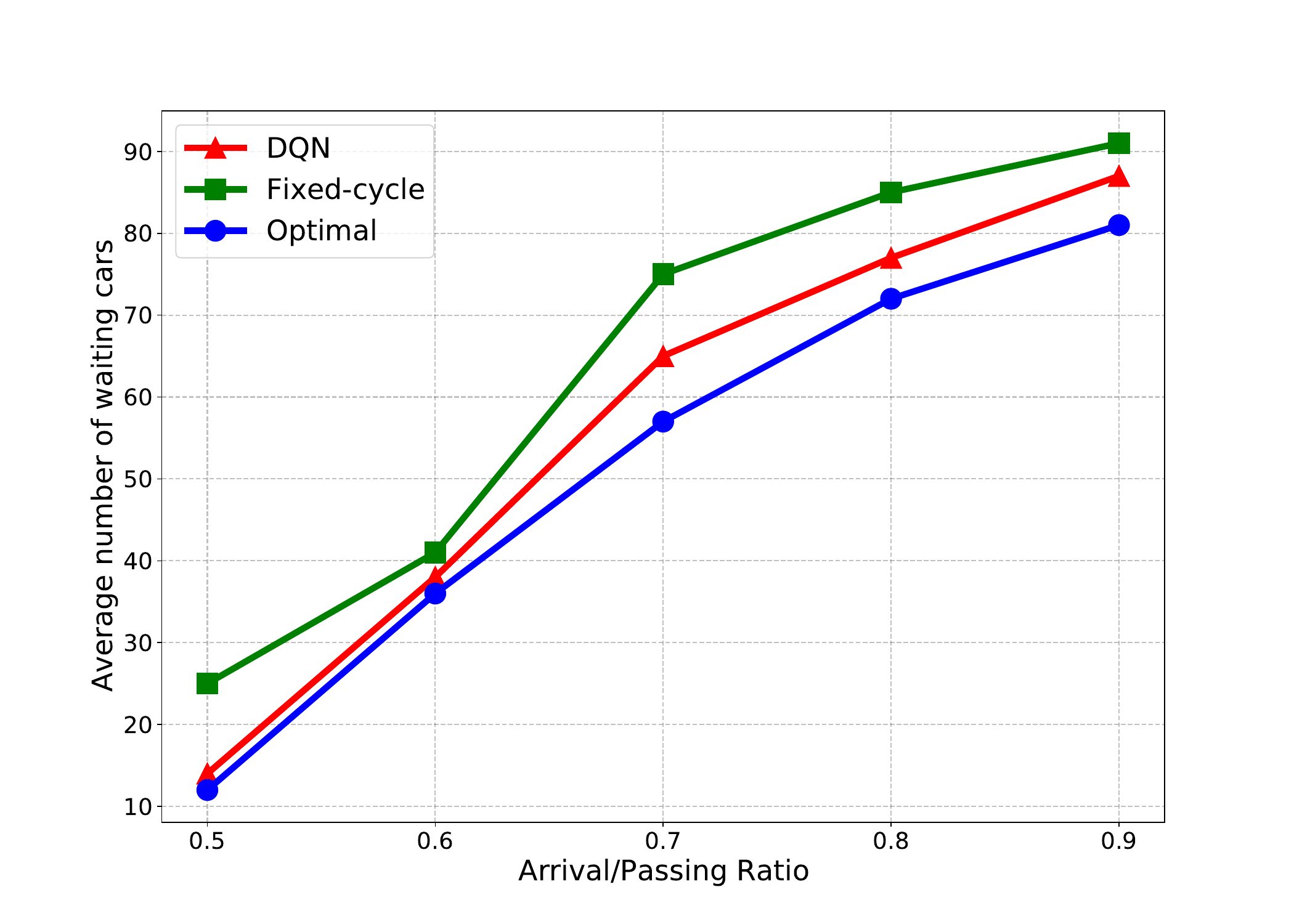}
			\caption{Comparison for a single intersection with different car arrival/passing ratio. 
			}
			\label{Fig:ArrivalParssingRatio}
		\end{minipage}%
		\hspace{0.3in}
		\begin{minipage}{0.46\textwidth}
			\centering
			\includegraphics[width=0.95\linewidth]{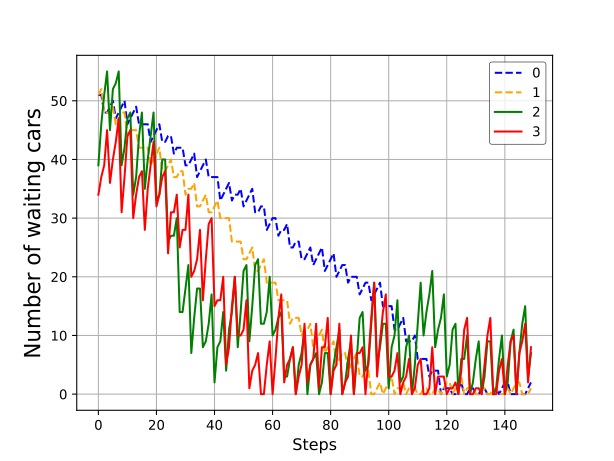}
			\caption{Stabilization of a single intersection with large initial queue length.
			}
			\label{fig:stable}
		\end{minipage}
		\label{fig:test}
	\end{figure*}

	\begin{figure*}[t]
		\centering
		\includegraphics[width=0.95\textwidth]{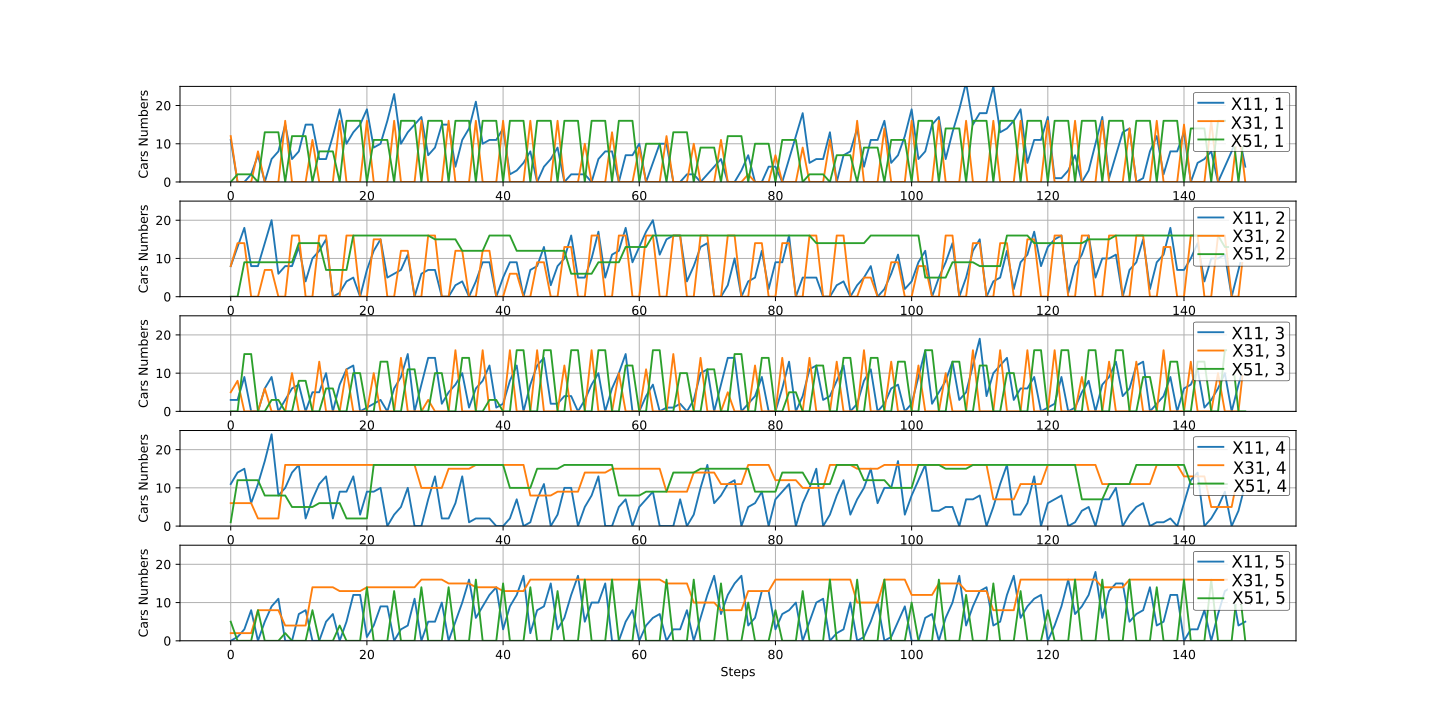}
		\caption{DDPG policy for a grid road network with $5 \times 10$ intersections.}
		\label{fig:DDPG_grid}
	\end{figure*}


	In this section, we first show experiment settings, and then present the results in a single road intersection, and finally demonstrate the results in a grid road network, especially the emergence of ``greenwave'' patterns.

	\subsection{Experiment Settings}

	In the experiments, we apply a 2-layer fully-connected neural network in DQN in the single intersection scenario, with size 400 and tanh as activation function. For the grid road network scenario, we apply a 4-layer fully-connected neural network in DDPG with size 600 for both the actor and critic. The actor needs to output near-binary values as action values, therefore, it has one modified sigmoid activation function for the last output layer; while the critic network has no activation function for the last output layer.

	The outputs of the actor network are clipped to binary values 0 and 1 indicating the traffic light state continuing and switching. We take an episode length of 150 steps of simulation for collecting the learning samples, and train both actor and critic networks with a batch size of 64 and $\gamma$ = 0.99. The Ornstein-Uhlenbeck noise \cite{horsthemke1980perturbation} with variance 0.3 is applied for explorations.

	The implementation of our algorithms uses the deep learning framework TensorFlow \cite{abadi2016tensorflow}. Our simulations are executed on a server with Linux OS, 200 GB memory, two Intel(R) Xeon(R) Gold 5118 CPUs@2.30 GHz, a Tesla V100-PCIE GPU.

	Let's discuss the traffic environment settings. The vehicle arrival rate indicates the number of vehicles arriving from outside of the road networks into the road networks on each intersection per time slot, which is set to be a random value with upper bound $C_a$ for avenue or $C_c$ for cross streets; and the vehicle passing rates represents the number of vehicles passing one intersection within one time step, and set to be fixed number of 16 and 4 for avenue and cross streets, respectively. The arrival/passing ratio is the ratio of arrival rate to passing rate. The comparison of average queue length (averaged sum of queueing cars on all roads over time) in a single intersection with respect to the arrival/passing ratio for different policies is shown in Fig.~\ref{Fig:ArrivalParssingRatio}.

	\subsection{Verifying Optimal Control in Single Road Intersection}

	As an initial validation benchmark, we first consider a single road intersection scenario as described in Subsection~\ref{Subsec:SingleRoadIntersection}. The reason for considering this toy scenario is that the state-action space is sufficiently small and the optimal policy can be computed using a conventional MDP approach such as policy iteration. We assume the numbers of arriving vehicles of both traffic flows in each time step are represented by the random variables, which are independent and Bernoulli distributed with parameter $P = 0.25$.

	In the single intersection, the optimal policy is a ``thresholding policy" \cite{hofri1987optimal}.
	And the optimal policy can be obtained using the policy iteration algorithm by the MDP toolbox \cite{Python_MDP}, which is a ``thresholding policy" \cite{hofri1987optimal}. Comparisons of the policies derived by the DQN algorithm, the fixed-cycle policy, and the optimal policy, are shown in Fig.~\ref{fig:SingleIntersection}. The result of the DQN policy coincides with that of the optimal policy. The training time using DQN is very short and can be ignored, since the state-action space is too small in a single road intersection. In particular, they match the optimal performance and exhibit a similar thresholding structure. This structural property was established in \cite{hofri1987optimal} for a strongly related two-queue dynamic optimization problem (with switch-over costs rather than switch-over times). Fig.~\ref{Fig:ArrivalParssingRatio} explores the effects of different arrival/passing ratio. As the car arrival/passing ratio increases, the queue length also increases. And the policy learned by DQN shows better effects than the fixed-cycle policy. Fig.~\ref{fig:stable} studies the stabilization with large initial queue lengths. The arrival/passing ratio is 0.5. The policy learned by DQN shows capability of stabilizing the traffic.

	\subsection{``Greenwave" in Grid Road Network} \label{subsection:DDPGPerformance}

	We now turn to a grid road network as Fig.~\ref{Fig:LinearRoadNetwork}. This is a more challenging scenario. It serves to examine the scalability properties of our DDPG algorithm. Moreover, it has more highly complex interactions arising from the vehicle flow along the avenue.

	We assume the numbers of arriving vehicles in eastern and western direction for the avenue in each time step are represented by random variables, which are independent and Bernoulli distributed with parameter $P_1 = 0.5$. The numbers of arriving vehicles in southern and northern directions on each of the $N$ cross streets in each time step are represented by random variables, which are independent and Bernoulli distributed with parameter $P_2 = 0.25$. Next, we evaluate the performance in a 5 $\times$ 10 grid road network.

	The car numbers with time steps are shown in Fig.~\ref{fig:DDPG_grid}. The conventional MDP approach such as policy iteration is computationally infeasible due to the state-action space explosion, and hence whether our algorithm is favorable cannot be assessed quantitatively. Therefore, we examine qualitative features to validate the intelligent behavior and evaluate its performance merit. In particular, we observe the emergence of ``greenwave'' patterns, even though such structural features are not explicitly prescribed in the optimization process. This emergent intelligence confirms the capability of our algorithms to learn favorable structural properties solely from passive observations.

	\section{Conclusion} \label{Sec:Conclusion}

	We have explored the potential for deep Q-network (DQN) and deep deterministic policy gradient (DDPG) algorithms to optimize real-time traffic light control policies in a single road intersection and a grid road network. As an initial benchmark, we established that the policy obtained from the DQN algorithm matches the optimal policy achieved by a conventional MDP approach in a single intersection. We subsequently evaluated the scalability of the DDPG algorithm in a grid road network, and demonstrated the emergence of the high-level intelligent behavior named as ``greenwave" policies, confirming its ability to learn desirable structural features. We also theoretically proved that the ``greenwave" policy is the optimal control policy under a specified traffic flow model (an avenue with multiple cross streets). The ``greenwave" patterns demonstrate that DRL algorithms output favorable solutions in a grid road network.

	In future research, we will consider the traffic flows with stochastic arrivals and departures of vehicles among all areas of a road network, and we intend to investigate locality properties and analyze how to design distributed coordination schemes for wide-scale deployment scenarios using new techniques, e.g., multi-agent DRL algorithms, hierarchical DRL algorithms, and edge computing.

	%

	\vspace{-0.40in}
	\begin{IEEEbiography}[{\includegraphics[width=1in,height=1.21in,clip,keepaspectratio]{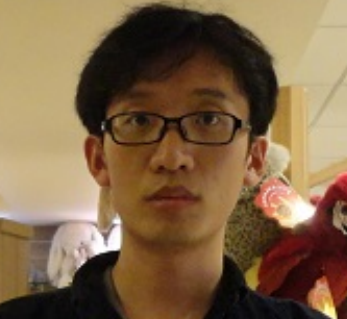}}]\\~~{\textbf{Ming~Zhu}} is with Institute of Automation, Chinese Academy of Sciences, Beijing 100190, China, also with the School of Artificial Intelligence, University of Chinese Academy of Sciences, Beijing 100049, China.

	His research interests are in the area of machine learning and quantum computing.
	\end{IEEEbiography}

	\vspace{-0.40in}
	\begin{IEEEbiography}[{\includegraphics[width=1.40in,height=1.07in,clip,keepaspectratio]{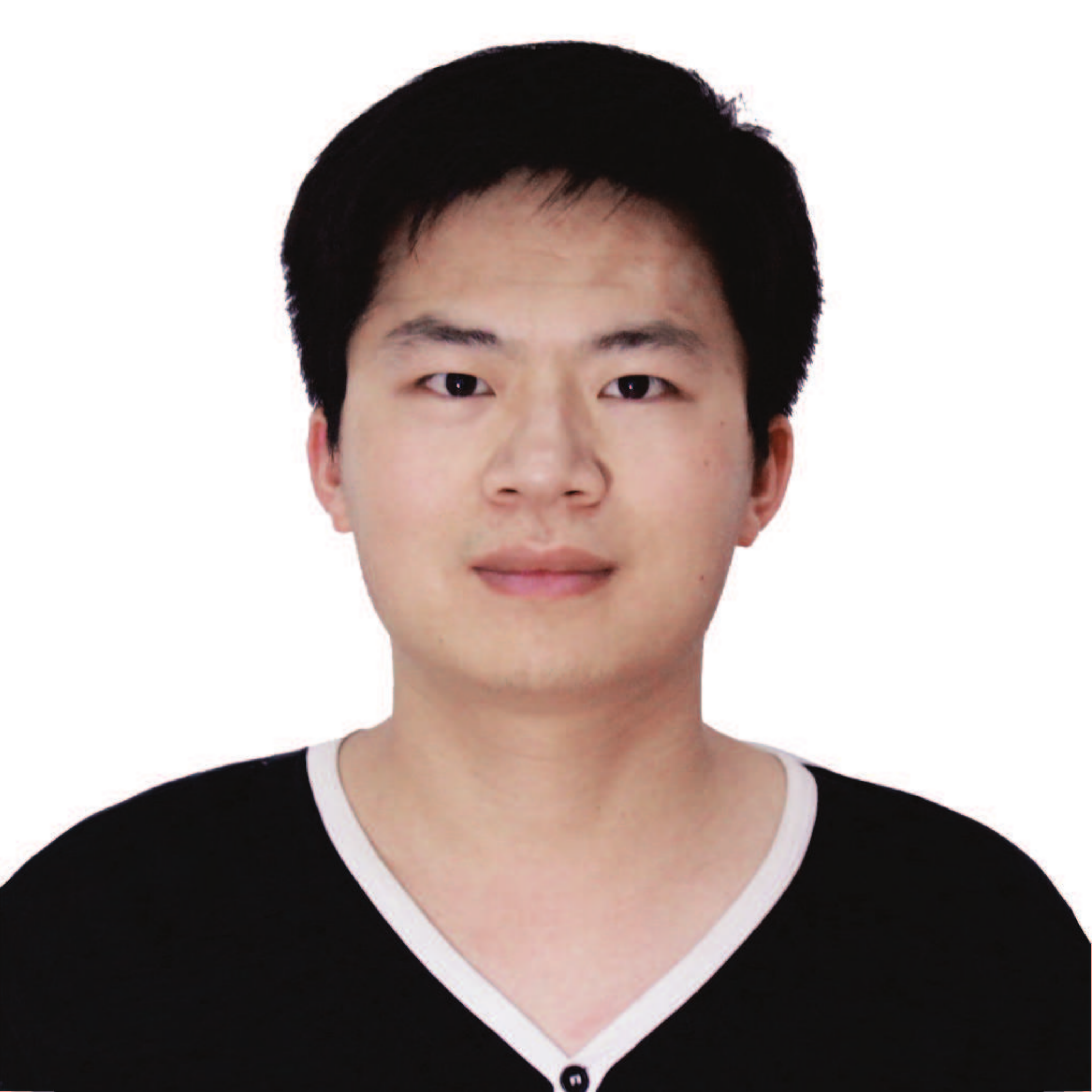}}]{Xiao-Yang~Liu} received M.S. degree in Electrical
	Engineering from Columbia University, New York, NY, USA, in 2018, and is currently a Ph.D. candidate at the Department of Electrical
	Engineering, Columbia University, New York, NY, USA.

	His research interests include tensor and tensor networks, high performance tensor computing, high performance tensor computing, deep learning, non-convex optimization, big data analysis and IoT applications.
	\end{IEEEbiography}
	
	\vspace{-0.40in}
	\begin{IEEEbiography}[{\includegraphics[width=1.40in,height=1.07in,clip,keepaspectratio]{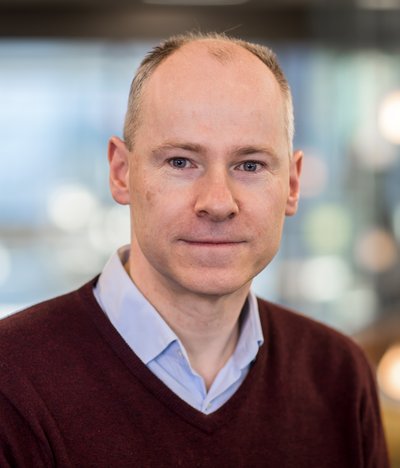}}]{Sem~Borst}
		Sem Borst received his MSc degree in applied mathematics from the University of Twente in 1990, and his PhD degree from Tilburg University in 1994.  In 1994, he was a visiting scholar at the Statistical Laboratory of the University of Cambridge, England.  In 1995, Sem joined the Mathematics of Networks and Systems research department of Bell Labs in Murray Hill, USA.  In addition to his position at TU/e, he maintains a (part-time) affiliation with Bell Labs. Sem has published over 190 papers in refereed journals and conference proceedings, and his H-index is 39.  He serves or has served on the editorial boards of several journals, such as ACM Transactions on Modeling and Performance Evaluation of Computing Systems, IEEE/ACM Transactions on Networking, Mathematical Methods of Operations Research and Queueing Systems, and served as program committee member of various conferences.
	\end{IEEEbiography}


	\vspace{-0.40in}
	\begin{IEEEbiography}[{\includegraphics[width=1.40in,height=1.07in,clip,keepaspectratio]{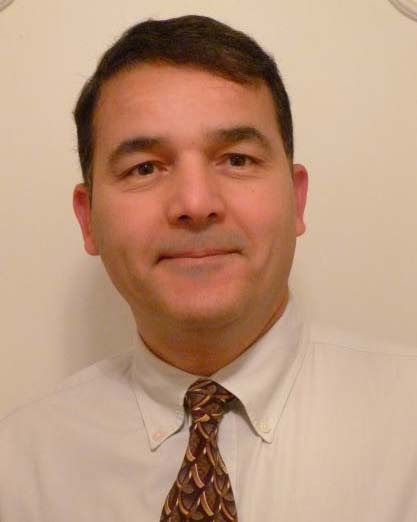}}]{Anwar~Walid}
		Anwar Walid (F’09) received the B.S. and M.S. degrees in electrical and computer engineering from New York University, New York, NY, USA, and the Ph.D. degree from Columbia University, New York, NY, USA.

		He is a Director of Network Intelligence and a Distributed Systems Research, and a Distinguished Member of the Research Staff with Nokia Bell Labs, Murray Hill, NJ, USA, where he also served as the Head of the Mathematics of System Research Department and the Director of University Research Partnerships. He is an Adjunct Professor with the Electrical Engineering Department, Columbia University. He has over 20 U.S. and international granted patents on various aspects of networking and computing. His current research interests include control and optimization of distributed systems, learning models and algorithms with applications to Internet of Things, digital health, smart transportations, cloud computing, and software defined networking.
		Dr. Walid was a recipient of IEEE and ACM Awards, including the 2017 IEEE Communications Society William R. Bennett Prize and the ACM SIGMETRICS/International Federation for Information Processing (IFIP) Performance Best Paper Award. He served as an Associate Editor for the IEEE/ACM Transactions on Cloud Computing, IEEE Network Magazine, and the IEEE/ACM Transactions on Networking. He served as the Technical Program Chair for IEEE INFOCOM, the General Chair of the 2018 IEEE/ACM Conference on Connected Health (CHASE), and the Guest Editor of the IEEE Internet of Things Journal “Special Issue on AI-Enabled Cognitive Communications and Networking for Internet of Things.” He is an elected member of the IFIP Working Group 7.3 and Tau Beta Pi Engineering Honor Society.
	\end{IEEEbiography}

	\begin{appendix}

		We first describe the traffic model. Then, we discuss the stability condition and the queue length evolution. Finally, we prove the optimality of the ``greenwave" control policy.

		\noindent{\textbf{A. Model Description}}

		Consider a grid road network topology with $N$ nodes,
		representing an avenue (i.e., an artery road) with $N$ cross streets. Avenue traffic arrives at node $1$ as a deterministic `fluid flow' with a constant rate $\lambda_0$, and must traverse all the nodes before it leaves the network after passing through node $N$. In addition, cross traffic arrives at each node $n$ as a deterministic fluid flow with a constant rate $\lambda_n$ and leaves the network after passing through that node, $n \in \{ 1, ..., N \}$. For stability, we assume that $\lambda_0 + \lambda_n < 1$ for all $n \in \{ 1, ..., N \}$.

		Avenue traffic can only flow through node~$n$ when the traffic light is green, either at unit rate when it is backed up at node~$n$, $n \in \{ 1, ..., N \}$, or otherwise at the outflow rate of the upstream node $n - 1$ (when $n > 1$) or the external arrival rate~$\lambda_0$ (when $n = 1$). Likewise, cross traffic can only flow through node~$n$ when the traffic light is red, either at unit rate when it is backed up, or otherwise at the arrival rate $\lambda_n$, $n \in \{ 1, ..., N \}$.


		Denote by $Y_n$ and $O_n$ the (fixed) durations of the yellow and orange periods at node~$n \in \{ 1, ..., N \}$, respectively. The lengths of the $k$-th green and red periods at node~$n \in \{ 1, ..., N \}$ can be controlled, and are denoted by $G_n^{(k)}$ and $R_n^{(k)}$, respectively. For compactness, denote by $U_n^{(k)} = G_n^{(k)} + R_n^{(k)} + Y_n + O_n$ the length of the $k$-th control cycle at node~$n \in \{ 1, ..., N \}$. Further, we define
		\begin{equation}
			G_n = \lim_{K \to \infty} \frac{1}{K} \sum_{k = 1}^{K} G_n^{(k)}, ~ R_n = \lim_{K \to \infty} \frac{1}{K} \sum_{k = 1}^{K} R_n^{(k)},
		\end{equation}
		as the long-term average durations of the green and red periods at node $n$, assuming these limits to exist. Denote by $U_n$ the long-term average control cycle length at node~$n$,
		\begin{equation}\label{Eqn:AverageDuration}
		U_n = G_n + R_n +Y_n + O_n, ~n \in \{ 1, ..., N \}.
		\end{equation}

		\noindent{\textbf{B. Stability Condition}}

		To ensure the stability of queues, it is required that
		\begin{equation}\label{Eqn:StabilityCondition}
		\frac{G_n}{U_n} \geq \lambda_0, ~~\frac{R_n}{U_n} \geq \lambda_n, ~~ n \in \{ 1, ..., N \},
		\end{equation}
		since the avenue traffic arrives continuously at rate $\lambda_0$ and can only pass node $n$ at unit rate during green periods at that node, and likewise, the cross traffic at node $n$ arrives continuously at rate $\lambda_n$ and can only pass at unit rate during red periods at that node.

		Throughout the remainder of the appendix, we assume that the inequations in \eqref{Eqn:StabilityCondition} are satisfied. Invoking \eqref{Eqn:AverageDuration}, the two inequations in \eqref{Eqn:StabilityCondition} may be equivalently expressed as
		\begin{equation}\label{stab2}
		\begin{split}
			G_n &\geq \frac{\lambda_0}{1 - \lambda_0} \left(R_n + Y_n + O_n\right), ~n \in \{ 1, ..., N \},\\
			R_n &\geq \frac{\lambda_n}{1 - \lambda_n} \left(G_n + Y_n + O_n\right), ~n \in \{ 1, ..., N \},
		\end{split}
		\end{equation}
		or in the following forms
		\begin{equation}\label{stab4}
		\begin{split}
			R_n &\leq \frac{1 - \lambda_0}{\lambda_0} G_n - Y_n - O_n, ~n \in \{ 1, ..., N \},\\
			G_n &\leq \frac{1 - \lambda_n}{\lambda_n} R_n - Y_n - O_n, ~n \in \{ 1, ..., N \}.
		\end{split}
		\end{equation}
		Considering that
		\begin{eqnarray}
			&& 1 - \frac{G_n}{U_n} - \frac{R_n}{U_n} = \frac{Y_n + O_n}{U_n} \overset{\eqref{Eqn:StabilityCondition}}{\leq} 1 - \lambda_0 - \lambda_n,
		\end{eqnarray}
		therefore,
		\begin{equation}
			U_n \geq \frac{Y_n + O_n}{1 - \lambda_0 - \lambda_n}. \label{Eqn:un_geq}
		\end{equation}
		Further, from \eqref{Eqn:StabilityCondition} and \eqref{Eqn:un_geq}, we obtain
		\begin{eqnarray}
			&& G_n \geq \frac{\lambda_0}{1 - \lambda_0 - \lambda_n}(Y_n + O_n), \label{Eqn:T0Geq}\\
			&& R_n \geq \frac{\lambda_n}{1 - \lambda_0 - \lambda_n}(Y_n + O_n). \label{Eqn:T1Geq}
		\end{eqnarray}
		Note that \eqref{Eqn:AverageDuration} and \eqref{Eqn:StabilityCondition} together yield
		\begin{equation}
			\lambda_0 + \lambda_n \leq 1 - \frac{Y_n + O_n}{U_n} < 1, \label{Eqn:Lambda0Lambdan<1}
		\end{equation}
		as assumed above, and the next lemma states some further useful inequations, whose proofs are provided in the appendix.

		\begin{lemm}
			\label{Lemma:LowerBound1}
			For all $n \in \{ 1, ..., N \}$,
			\begin{equation}
				\frac{(G_n + Y_n + O_n)^2}{U_n} \geq \frac{\lambda_0 (p_n + 1)^2 (Y_n + O_n)}{p_n},
				\label{Eqn:T0}
			\end{equation}
			\begin{equation}
				\frac{(R_n + Y_n + O_n)^2}{U_n} \geq \frac{\lambda_n (q_n + 1)^2 (Y_n + O_n)}{q_n},
				\label{Eqn:T1}
			\end{equation}
			\begin{equation}
				p_n = \max\{1, \frac{\lambda_0}{1 - \lambda_0 - \lambda_n}\},~~ q_n = \max \{ 1, \frac{\lambda_n}{1 - \lambda_0 - \lambda_n} \}. \nonumber
			\end{equation}
			The above inequations are strict unless $G_n = p_n (Y_n + O_n)$ and $R_n = q_n (Y_n + O_n)$, respectively.
		\end{lemm}

		\begin{proof}
			We only show the proof of \eqref{Eqn:T0} since the proof of \eqref{Eqn:T1} is similar. We have
			\begin{equation}
				\frac{(G_n + Y_n + O_n)^2}{U_n} \overset{\eqref{Eqn:StabilityCondition}}{\geq} \frac{\lambda_0 (G_n + Y_n + O_n)^2}{G_n}. \nonumber
			\end{equation}
			Differentiation shows that the right-hand size is increasing in $G_n$ for $G_n \geq Y_n + O_n$, and decreasing in $G_n$ for $G_n < Y_n + O_n$. Thus, combining the lower bound of $G_n$ in \eqref{Eqn:T0Geq}, the minimum value of the right-hand side is attained for $G_n = p_n (Y_n + O_n)$ with $p_n = \max\{1, \frac{\lambda_0}{1 - \lambda_0 - \lambda_n}\}$, and equals
			\begin{equation}
				\frac{\lambda_0 (p_n + 1)^2 (Y_n + O_n)}{p_n}, \nonumber
			\end{equation}
			yielding the lower bound~\eqref{Eqn:T0}. It also follows that the lower bound is strict unless $G_n = p_n (Y_n + O_n)$.
		\end{proof}

		\noindent{\textbf{C. Queue Evolution}}

		Denote by $\widehat{\Phi}_n(b, t)$ and $\widehat{\Psi}_n(b, t)$ the queue length of the avenue and the cross traffic at node~$n$ at time~$t$ with the traffic flow begins from the time~$b$, respectively. Let $\Phi_n(b, t)$ and $\Psi_n(b, t)$ be the lower bounds of $\widehat{\Phi}_n(b, t)$ and $\widehat{\Psi}_n(b, t)$. If the traffic flow starts from the time~$b = 0$, $\widehat{\Phi}_n(b, t)$, $\widehat{\Psi}_n(b, t)$, $\Phi_n(b, t)$ and $\Psi_n(b, t)$ are simplified as $\widehat{\Phi}_n(t)$, $\widehat{\Psi}_n(t)$, $\Phi_n(t)$ and $\Psi_n(t)$, respectively. Denote by $\overline{G}_n(b, t)$ and $\overline{R}_n(b, t)$ the amount of non-green and non-red time at node~$n \in \{ 1, ..., N \}$, during the time interval $(b, t]$, respectively. Now, observe that $\widehat{\Phi}_1(b, t)$ increases at rate~$\lambda_0$ during non-green periods at node~1 and decreases at rate $1 - \lambda_0$ during green periods. Thus,
		\begin{equation}
			\begin{split}
				\widehat{\Phi}_1(b, t) & \geq \lambda_0 \overline{G}_1(b, t) - (1 - \lambda_0) (t - b - \overline{G}_1(b, t)) \\
				& =\overline{G}_1(b, t) - (1 - \lambda_0) (t - b), \label{Eqn:Phi1stGeq}
			\end{split}
		\end{equation}
		for all $b \leq t$, yielding
		\begin{equation}
			\widehat{\Phi}_1(b, t) \geq \sup_{s \leq t} \{ \overline{G}_1(b, t) - (1 - \lambda_0) (t - b) \},
		\end{equation}
		and assuming $\widehat{\Phi}_1(0) = 0$, we in fact have
		\begin{equation}
			\Phi_1(t) = \sup_{s \leq t} \{ \overline{G}_1(b, t) - (1 - \lambda_0) (t - b) \}.
		\end{equation}

		\begin{remark} \label{Remark:Phi1t}
		Without loss of generality, we assume that each control cycle ends with a green period, and thus starts with a yellow period. Denote by $B^{(k)}$ the start time of the $k$-th control cycle at node~1. It can be shown that
		\begin{eqnarray}
			&& \hspace{-0.5in} \Phi_1(t) = \sup_{s \leq t} \{\overline{G}_1(b, t) - (1 - \lambda_0) (t - b)\} \nonumber \\
			&& \hspace{-0.5in} = \max\{\sup_{l \in \{ 1, ..., \kappa(t) \} } \{\overline{G}_1(B^{(l)}, t) \! - \! (1 \! - \! \lambda_0) (t \! - \! B^{(l)})\}, \! 0\}, \label{Eqn:Phi1tEqual}
		\end{eqnarray}
		where $\kappa(t)$ is the maximum number of cycles until time $t$.
		\end{remark}
		First, if $b \in [B^{(l-1)} + \overline{G}_1^{(l-1)}, B^{(l)}]$, thus $\overline{G}_1(b, t) = \overline{G}_1(B^{(l)}, t)$, and $\overline{G}_1(b, t) - (1 - \lambda_0) (t - b) \leq \overline{G}_1(B^{(l)}, t) - (1 - \lambda_0) (t - B^{(l)})$. Secondly, if $b \in [B^{(l)}, B^{(l)} + \overline{G}_1^{(l)}]$, thus $\overline{G}_1(b, t) = \overline{G}_1(B^{(l)}, t) - (b - B^{(l)})$ with $b - B^{(l)} \geq 0$, and $\overline{G}_1(b, t) - (1 - \lambda_0) (t - b) \leq \overline{G}_1(B^{(l)}, t) - (1 - \lambda_0) (t - B^{(l)})$. The above two cases are finally merged as $\overline{G}_1(b, t) - (1 - \lambda_0) (t - b) \leq \overline{G}_1(B^{(l)}, t) - (1 - \lambda_0) (t - B^{(l)})$ for all $b \in [B^{(l-1)} + \overline{G}_1^{(l-1)}, B^{(l)} + \overline{G}_1^{(l)}]$. We see that, $\overline{G}_1(b, t) - (1 - \lambda_0) (t - b) \leq \overline{G}_1(B^{(l)}, t) - (1 - \lambda_0) (t - B^{(l)})$ is always satisfied for each cycle $l \in \{ 1, ..., \kappa(t) \}$.

		Likewise, at node $n$, $\widehat{\Psi}_n(b, t)$ increases at rate~$\lambda_n$ during non-red periods and decreases at rate~$1 - \lambda_n$ during red periods. Thus,
		\begin{eqnarray}
			\hspace{-0.3in} \widehat{\Psi}_n(b, t) \! &\geq& \! \lambda_n \overline{R}_n(b, t) \! - \! (1 - \lambda_n) (t - b - \overline{R}_n(b, t)) \nonumber \\
			\! &=&\! \overline{R}_n(b, t) - (1 - \lambda_n) (t - b),
		\end{eqnarray}
		for all $b \leq t$, yielding
		\begin{equation}
			\widehat{\Psi}_n(t) \geq \sup_{s \leq t} \{\overline{R}_n(b, t) - (1 - \lambda_n) (t - b)\},
		\end{equation}
		and assuming $\widehat{\Psi}_n(0) = 0$, we in fact have
		\begin{equation}
			\Psi_n(t) = \sup_{s \leq t} \{\overline{R}_n(b, t) - (1 - \lambda_n) (t - b)\}.
		\end{equation}

		The dynamics of $\widehat{\Phi}_n(t)$ are similar but also different, since the inflow rate of avenue traffic at node~$n$ is not fixed, and equal to the outflow rate of the upstream node $n - 1$. Denote by $\eta_n(t)$ the outflow rate of avenue traffic at node $n$ at time~$t$, and note that $\eta_n(t) = 0$ during non-green periods, $\eta_n(t) = 1$ during green periods when $\widehat{\Phi}_n(t) > 0$, and $\eta_n(t) = \eta_{n-1}(t)$ during green periods, and $\eta_n(t) = 0$ when $\widehat{\Phi}_n(t) = 0$ in all periods. Assuming $\widehat{\Phi}_n(0) = 0$, we have
		\begin{equation}
			\Phi_n(t) = \sup_{s \leq t} \{\int_{x = 0}^{t} \eta_{n - 1}(x) dx - (t - b - \overline{G}_n(b, t))\}.
		\end{equation}
		Define
		\begin{eqnarray}
			&& \Phi_n = \lim\sup_{T \to \infty} \frac{1}{T} \int_{t = 0}^{T} \Phi_n(t) dt, \nonumber \\
			&& \Psi_n = \lim\sup_{T \to \infty} \frac{1}{T} \int_{t = 0}^{T} \Psi_n(t) dt, \nonumber
		\end{eqnarray}
		as the long-term average queue length of the avenue and the cross traffic at node~$n$, respectively. The lower bound of $\Phi_1(t)$ is obtained by Lemma~\ref{Lemma:Phi1tGeq}.

		\begin{lemm} \label{Lemma:Phi1tGeq}
		\begin{eqnarray}
			&& \hspace{-0.62in} \Phi_1(t) \geq \sum_{k = 1}^{\infty} \max\{\overline{G}_1^{(k)}(t) - (1 - \lambda_0) U^{(k)}, 0\},  \label{Eqn:Phi1tGeq}
		\end{eqnarray}
		which is further deduced as
		\begin{eqnarray}
			&& \hspace{-0.3in} \Phi_1(t) \geq \sum_{k = 1}^{\infty} \max\{\overline{G}_1^{(k)}(t) - (1 - \lambda_0) |t - B^{(k)}|, 0\}. \label{Eqn:Phi1tGeq__}
		\end{eqnarray}
		\end{lemm}

		\begin{proof}
			Without loss of generality we assume that each control cycle ends with a green period (and thus starts with a yellow period). For compactness, denote by $\overline{G}_1^{(k)} = R_1^{(k)} + Y_1 + O_1$ the amount of non-green time during the $k$-th control cycle at node~$1$. Denote by $B^{(k)}$ the start time of the $k$-th control cycle at node~$1$. Define
			\begin{equation}
				\overline{G}_1^{(k)}(t) = \max\{\min\{\overline{G}_1^{(k)}, t - B^{(k)}\}, 0\}, \nonumber
			\end{equation}
			as the amount of non-green time during the $k$-th control cycle at node~$1$ that has elapsed by time~$t$, or equivalently,
			\begin{equation} \label{Eqn:bar_g1}
			\overline{G}_1^{(k)}(t) = \left\{\begin{array}{ll}
												 0, & ~\text{if}~ t \leq B^{(k)}, \\
												 t - B^{(k)}, & ~\text{if}~ B^{(k)} \leq t \leq B^{(k)} + \overline{G}_1^{(k)}, \\
												 \overline{G}_1^{(k)}, & ~\text{if}~ t \geq B^{(k)} + \overline{G}_1^{(k)}. \end{array}\right.
			\end{equation}

			Define $\overline{G}_1^{(k)}(b, t) = \overline{G}_1^{(k)}(t) - \overline{G}_1^{(k)}(b)$ as the amount of non-green time during the $k$-th control cycle that occurs in the time interval $(b, t]$ at node 1, and note that $\overline{G}_1(b, t) = \sum_{k = 1}^{\infty} \overline{G}_1^{(k)}(b, t)$ is total amount of non-green time that occurs in the time interval $(b, t]$ at node 1. Let $U^{(k)}$ be the duration of the $k$-th cycle.

			\begin{eqnarray}
				&& \hspace{-0.38in} \Phi_1(t)  \overset{\eqref{Eqn:Phi1tEqual}}{=} \sup_{l \in \{ 1, ..., \kappa(t) \} } \{\overline{G}_1(B^{(l)}, t) \! - \! (1 \! - \! \lambda_0) (t \! - \! B^{(l)})\}  \nonumber \\
				&& \geq \overline{G}_1^{(1)}(B^{(1)}, t) -  (1 - \lambda_0) (t-  B^{(1)})  \nonumber \\
				&& = \sum_{k = 1}^{\kappa(t)} [\overline{G}_1^{(k)}(t) - (1 - \lambda_0) U^{(k)}]  \label{Eqn:Phi1tGeq+} \\
				&& \geq \sum_{k = 1}^{\kappa(t)} [\overline{G}_1^{(k)}(t) - (1 - \lambda_0) (t - B^{(k)})].   \label{Eqn:Phi1tGeq_}
			\end{eqnarray}
			When $k \geq \kappa(t) + 1$, $\overline{G}_1^{(k)}(t) = 0$, and $\overline{G}_1^{(k)}(t) - (1 - \lambda_0) U^{(k)} < 0$. Therefore, combining \eqref{Eqn:Phi1tGeq+}, \eqref{Eqn:Phi1tGeq} is obtained. When $k = \kappa(t) + 1$, $t -  B^{(k)} = 0$. When $k \geq \kappa(t) + 2$, $t -  B^{(k)} < 0$, and $|t -  B^{(k)}| > 0$. Therefore, $\sum_{k = 1}^{\kappa(t)} [\overline{G}_1^{(k)}(t) - (1 - \lambda_0) (t - B^{(k)})] = \sum_{k = 1}^{\infty} \max\{\overline{G}_1^{(k)}(t) - (1 - \lambda_0) |t -  B^{(k)}|, 0\}$. Combining \eqref{Eqn:Phi1tGeq_}, \eqref{Eqn:Phi1tGeq__} is obtained.

			Now observe that when $B^{(k)} \leq t \leq B^{(k)} + \frac{1}{1 - \lambda_0} \overline{G}_1^{(k)}$, we have
			\begin{eqnarray}
				&& \hspace{-0.29in} \max\{\overline{G}_1^{(k)}(B^{(k)}, t) - (1 - \lambda_0) |t - B^{(k)}|, 0\} \nonumber \\
				&& \hspace{-0.29in} = \overline{G}_1^{(k)}(B^{(k)}, t) - (1 - \lambda_0) (t - B^{(k)}) \nonumber \\
				&& \hspace{-0.29in} =\left\{ \begin{array}{ll}
												 \lambda_0 (t - B^{(k)}) - \overline{G}_1^{(k)} (B^{(k)}) , &\text{if}~B^{(k)} \leq t \leq B^{(k)} + \overline{G}_1^{(k)}, \\
												 \overline{G}_1^{(k)} - (1 - \lambda_0) (t - B^{(k)}) & \\
												 ~~~~ - \overline{G}_1^{(k)} (B^{(k)}) , & \text{if}~B^{(k)} + \overline{G}_1^{(k)} \leq t \\
												 & ~~~~ \leq B^{(k)} + \frac{1}{1 - \lambda_0} \overline{G}_1^{(k)}.
				\end{array}
				\right. \nonumber
			\end{eqnarray}
		\end{proof}

		\begin{proposition}
			\label{Proposition:QueueLengthLowerBound}
			For $n \in \{ 1, ..., N \}$, we have
			\begin{equation}
				\begin{split}
					\Phi_1 &\geq
					\frac{\lambda_0 (R_1 + Y_1 + O_1)^2}{2 (1 - \lambda_0) U_1},  \\
					\Psi_n &\geq
					\frac{\lambda_n (G_n + Y_n + O_n)^2}{2 (1 - \lambda_n) U_n}.
				\end{split}
			\end{equation}
		\end{proposition}

		\begin{proof}

			Define
			\begin{equation} \nonumber
			E^{(k)} = \overline{G}_1^{(k)} - (1 - \lambda_0) U^{(k)},
			\end{equation}
			and
			\begin{equation} \nonumber
			F^{(k)} = \sup_{l \in \{ 1, ..., k \}} \sum_{m = l}^{k} E^{(m)} ,
			\end{equation}
			and consider $k$ such that $F^{(k)} = 0$, i.e., $\sum_{m = l}^{k} E^{(m)} \leq 0$ for all $l \in \{ 1, ..., k \}$,
			implying
			\begin{equation}
				\begin{split}
					\overline{G}_1^{(l)} - (1 - \lambda_0) \sum_{m = l}^{k} U^{(m)} & \leq \sum_{m = l}^{k} \overline{G}_1^{(m)} - (1 - \lambda_0) \sum_{m = l}^{k} U^{(m)} \\
					&= \sum_{m = l }^{k} E^{(m)} \leq 0,
				\end{split}
			\end{equation}
			i.e.,
			\begin{equation} \nonumber
			\overline{G}_1^{(l)} \leq (1 - \lambda_0) \sum_{m = l }^{k} U^{(m)},
			\end{equation}
			and thus
			\begin{equation} \nonumber
			B^{(l)} + \frac{1}{1 - \lambda_0} \overline{G}_1^{(l)} \leq B^{(l)} + \sum_{m = l}^{k} U^{(m)} = B^{(k+1)}.
			\end{equation}
			Now consider
			\begin{equation}
				\begin{split}
					& \int_{t = 0}^{B^{(k+1)}} \Phi_1(t) dt \\
					& \overset{\eqref{Eqn:Phi1tGeq__}}{\geq} \int_{t = 0}^{B^{(k+1)}} \sum_{l = 1}^{\kappa(t)} \max\{\overline{G}_1^{(l)}(t) - (1 - \lambda_0) |t - B^{(l)}| , 0\}  dt \\
					&=\sum_{l = 1}^{\kappa(t)} \int_{t = 0}^{B^{(k+1)}} \max\{\overline{G}_1^{(l)}(t) - (1 - \lambda_0) |t - B^{(l)}| , 0\} dt .
				\end{split}
			\end{equation}

			Note that
			\begin{eqnarray}
				&& \hspace{-0.3in} \max\{\overline{G}_1^{(l)}(t) - (1 - \lambda_0) |t - B^{(l)}| , 0\} = \nonumber \\
				&& \hspace{-0.3in} \begin{cases}
									   0, & \text{if}~t \leq B^{(l)}, \\
									   \lambda_0 (t - B^{(l)}), & \text{if}~B^{(l)} \leq t \leq B^{(l)} + \overline{G}_1^{(l)}, \\
									   \overline{G}_1^{(l)} - (1 - \lambda_0) (t - B^{(l)}), & \text{if}~B^{(l)} + \overline{G}_1^{(l)} \leq t \\
									   & ~~~~ \leq B^{(l)} + \frac{1}{1 - \lambda_0} \overline{G}_1^{(l)}, \\
									   0, & \text{if}~t \geq B^{(l)} + \frac{1}{1 - \lambda_0} \overline{G}_1^{(l)}.
				\end{cases}
			\end{eqnarray}

			Thus
			\begin{eqnarray}
				&&\int_{t = 0}^{B^{(k+1)}} \max\{\overline{G}_1^{(l)}(t) - (1 - \lambda_0) t, 0\} \nonumber \\
				&&= \int_{t = B^{(l)}}^{B^{(l)} + \overline{G}_1^{(l)}} \lambda_0 (t - B^{(l)}) dt  \nonumber \\
				&&~~~~ + \int_{t = B^{(l)} + \overline{G}_1^{(l)}}^{B^{(l)} + \frac{1}{1 - \lambda_0} \overline{G}_1^{(l)}}[\overline{G}_1^{(l)} - (1 - \lambda_0) (t - B^{(l)})] dt\nonumber \\
				&&= \int_{x = 0}^{\overline{G}_1^{(l)}} \lambda_0 x dx + \int_{x = \overline{G}_1^{(l)}}^{\frac{1}{1 - \lambda_0} \overline{G}_1^{(l)}} [\overline{G}_1^{(l)} - (1 - \lambda_0) x] dx \nonumber \\
				&&= \left[\frac{1}{2} \lambda_0 x^2\right]_{0}^{\overline{G}_1^{(l)}} \! + \! \frac{\lambda_0 \left(\overline{G}_1^{(l)}\right)^2}{1 - \lambda_0} \! - \! \left[\frac{1}{2} (1 - \lambda_0)x^2\right]_{\overline{G}_1^{(l)}}^{\frac{1}{1 - \lambda_0} \overline{G}_1^{(l)}} \nonumber \\
				&&= \frac{1}{2} \lambda_0 \left(\overline{G}_1^{(l)}\right)^2 + \frac{\lambda_0 \left(\overline{G}_1^{(l)}\right)^2}{1 - \lambda_0} - \frac{1}{2} \frac{\left(\overline{G}_1^{(l)}\right)^2}{1 - \lambda_0} \nonumber \\
				&& ~~~ + \frac{1}{2} (1 - \lambda_0) \left( \overline{G}_1^{(l)} \right)^2 \nonumber \\
				&&= \frac{\lambda_0 \left(\overline{G}_1^{(l)}\right)^2}{2 (1 - \lambda_0)} .
			\end{eqnarray}

			We obtain that
			\begin{equation} \nonumber
			\int_{t = 0}^{B^{(k+1)}} \Phi_1(t) dt \! \geq
			\! \sum_{l = 1}^{\kappa(t)} \frac{\lambda_0 \left(\overline{G}_1^{(l)}\right)^2}{2 (1 - \lambda_0)} \! = \!
			\frac{\lambda_0}{2 (1 - \lambda_0)} \sum_{l = 1}^{\kappa(t)} \left( \overline{G}_1^{(l)} \right)^2.
			\end{equation}

			Now observe that
			\begin{eqnarray}
				&& \lim_{K \to \infty} \frac{1}{K} \sum_{k = 1}^{K} E^{(k)} = R_1 + Y_1 + O_1 - (1 - \lambda_0) U_1 \nonumber \\
				&& = - G_1 + \lambda_0 U_1 < 0,
			\end{eqnarray}
			which means that the subsequence of $k$ for which $F^{(k)} = 0$ is infinite, so that
			\begin{eqnarray}
				\Phi_1 &=& \lim\sup_{T \to \infty} \frac{1}{T} \int_{t = 0}^{T} \Phi_1(t) dt \nonumber \\
				&\geq&  \lim\sup_{k \to \infty} \frac{1}{B^{(k+1)}} \int_{t = 0}^{B^{(k+1)}} \Phi_1(t) dt \nonumber \\
				& \geq&  \lim\sup_{k \to \infty} \frac{1}{B^{(k+1)}} \frac{\lambda_0}{2 (1 - \lambda_0)} \sum_{l = 1}^{k} \left(\overline{G}_1^{(l)}\right)^2 \nonumber \\
				&=&  \lim\sup_{k \to \infty} \frac{1}{\sum_{l = 1}^{k} U^{(l)}} \frac{\lambda_0}{2 (1 - \lambda_0)} \sum_{l = 1}^{k} \left(\overline{G}_1^{(l)}\right)^2 \nonumber \\
				&=& \frac{\lambda_0}{2 (1 - \lambda_0)} \lim\sup_{k \to \infty} \frac{1}{\frac{1}{k} \sum_{l = 1}^{k} U^{(l)}} \frac{1}{k} \sum_{l = 1}^{k} \left(\overline{G}_1^{(l)}\right)^2 \nonumber \\
				&=& \frac{\lambda_0}{2 (1 - \lambda_0)} \frac{(R_1 + Y_1 + O_1)^2}{U_1},
			\end{eqnarray}
			as claimed.
		\end{proof}

		Combining Lemma~\ref{Lemma:LowerBound1} and Proposition~\ref{Proposition:QueueLengthLowerBound}, we obtain the lower bounds
		\begin{equation} \nonumber
		\hspace{-0.7in}
		\begin{split}
			\Phi_1 &\geq \frac{\lambda_0}{2 (1 - \lambda_0)}
			\frac{\lambda_0 (p_1 + 1)^2 (Y_1 + O_1)}{p_1} \\
			&= \frac{\lambda_0^2 (p_1 + 1)^2 (Y_1 + O_1)}{2 p_1 (1 - \lambda_0)},
		\end{split}
		\end{equation}
		\begin{equation} \nonumber
		\begin{split}
			\Psi_n &\geq \frac{\lambda_n}{2 (1 - \lambda_n)}
			\frac{\lambda_n (q_n + 1)^2 (Y_n + O_n)}{q_n} \\
			&= \frac{\lambda_n^2 (q_n + 1)^2 (Y_n + O_n)}{2 q_n (1 - \lambda_n)},
			\hspace*{.4in} n \in \{ 1, ..., N \},
		\end{split}
		\end{equation}
		and conclude that the two inequations are strict unless $G_1 = p_1 (Y_1 + O_1)$ and $R_n = q_n (Y_n + O_n)$, respectively.

		We henceforth make an mild assumption that the fixed durations of the yellow and orange periods are identical at all nodes, i.e., $Y_n \equiv Y$ and $O_n \equiv O$ for all $n \in \{ 1, ..., N \}$. The above lower bounds are simplified as
		\begin{eqnarray}
			&& \Phi_1 \geq \frac{\lambda_0^2 (p_1 + 1)^2 (Y + O)}{2 p_1 (1 - \lambda_0)},\label{Eqn:Phi1Geq} \\
			&& \Psi_n \geq \frac{\lambda_n^2 (q_n + 1)^2 (Y + O)}{2 q_n (1 - \lambda_n)}, \label{Eqn:PsinGeq}
		\end{eqnarray}
		which are strict unless $G_1 = p_1 (Y_1 + O_1)$ and $R_n = q_n (Y_n + O_n)$, respectively.


		\noindent{\textbf{D. Optimality of the ``Greenwave" Control Policy}}

		We now consider a greenwave control policy which synchronizes the states across all the nodes, i.e., uses simultaneous green, orange, red and yellow periods. Denote by $\lambda_{\max} = \max_{n \in \{ 1, ..., N \}} \lambda_n$ the maximum arrival rate over all the cross streets. We select a common length of the green period
		\begin{equation}
			G_n \equiv G(\delta) = \frac{\lambda_0 (1 + \delta) (Y + O)}{1 - \lambda_0 - \lambda_{\max}}, \label{Eqn:Gdelta}
		\end{equation}
		and a common length of the red period
		\begin{equation}
			R_n \equiv R(\delta) = \frac{\lambda_{\max} (1 + \delta) (Y + O)}{1 - \lambda_0 - \lambda_{\max}}, \label{Eqn:Rdelta}
		\end{equation}
		where $\delta\geq0$ is a small scalability factor based on the lower bounds \eqref{Eqn:T0Geq} and \eqref{Eqn:T1Geq}. Then \eqref{Eqn:Gdelta} and \eqref{Eqn:Rdelta} yield a common control cycle length
		\begin{equation}
			U_n \equiv  U(\delta) = \frac{(1 + \delta (\lambda_0 + \lambda_{\max})) (Y + O)}{1 - \lambda_0 - \lambda_{\max}}. \label{Eqn:Udelta}
		\end{equation}

		The next proposition plays an instrumental role in establishing the optimality of the greenwave control policy.

		\begin{proposition}
			\label{greenwavecontrol}
			The long-term average queue length under the greenwave control policy is equal to
			\begin{eqnarray}
				&& \hspace{-0.4in} \Phi_1(\delta) = \frac{\lambda_0 (R(\delta) + Y + O)^2}{2 (1 - \lambda_0) U(\delta)}  \nonumber \\
				&& \hspace{-0.4in} \overset{\text{\eqref{Eqn:Rdelta} \eqref{Eqn:Udelta}}}{=} \frac{\lambda_0 (1 - \lambda_0 + \delta \lambda_{\max})^2 (Y + O)}{2 (1 - \lambda_0) (1 - \lambda_0 - \lambda_{\max}) [1 + \delta (\lambda_0 + \lambda_{\max})]},  \\
				&& \hspace{-0.4in} \Psi_n(\delta) = \frac{\lambda_n (G(\delta) + Y + O)^2}{2 (1 - \lambda_n) U(\delta)}  \nonumber \\
				&& \hspace{-0.4in} \overset{\text{\eqref{Eqn:Gdelta} \eqref{Eqn:Udelta}}}{=} \frac{\lambda_n (1 + \delta \lambda_0 - \lambda_{\max})^2 (Y + O)} {(1 - \lambda_n) (1 - \lambda_0 - \lambda_{\max}) [1 + \delta (\lambda_0 + \lambda_{\max})]} ,
			\end{eqnarray}
			for $n \in \{ 1, ..., N \}$, and $\Phi_n(\delta) = 0$ for all $n \geq 2$.
		\end{proposition}

		\begin{proof}

			Note that $\{ \Phi_n \}_{n \in \{ 1, ..., N \}}$ increases by $\lambda_0 U(\delta)$  over the course of a cycle, and must have decreased by $G(\delta)$ over the preceding cycle if it is strictly positive at the end of a green period. Thus, if $\{ \Phi_n \}_{n \in \{ 1, ..., N \}}$ is strictly positive at the end of a green period, it must decrease at the end of the preceding green period by
			\begin{equation}
				\begin{split}
					&\Delta_{\Phi_n} = G(\delta) - \lambda_0 U(\delta)  \\
					&\overset{\text{\eqref{Eqn:Gdelta} \eqref{Eqn:Udelta}}}{=} \frac{\lambda_0 (1 + \delta) (Y + O)}{1 - \lambda_0 - \lambda_{\max}} - \frac{\lambda_0 [1 + \delta (\lambda_0 + \lambda_{\max})] (Y + O)}{1 - \lambda_0 - \lambda_{\max}}  \\
					&= \lambda_0 \delta (Y + O)  \\
					&> 0.
				\end{split}
			\end{equation}

			We conclude that for any initial state $\{ \Phi_n \}_{n \in \{ 1, ..., N \}}$ must eventually become 0 at the end of some green period, say at time~$t_0$. Now observe that $\{ \Phi_n \}_{n \in \{ 2, ..., N \}}$ can only receive traffic (at most at unit rate) during a green period, but will also serve traffic at unit rate during a green period, and hence must remain zero from that point onward.

			$\Phi_1$ increases at rate $\lambda_0$ from time $t_0$ to time $t_0 + R(\delta) + Y + O$, and decreases at rate $1 - \lambda_0$ from time $t_0 + R(\delta) + Y + O$ to time $t_0 + R(\delta) + Y + O + \frac{\lambda_0}{1 - \lambda_0} \left( R(\delta) + Y + O \right) = t_0 + \frac{1}{1 - \lambda_0} \left( R(\delta) + Y + O \right) < t_0 +  U(\delta)$, to become zero at that point and remain zero until $t_0 +  U(\delta)$. Thus
			\begin{equation} \nonumber
			\begin{split}
				& \Phi_1(t_0 + t) = \\
				& \left\{ \begin{array}{ll} \lambda_0 t, & \text{if}~t \in [0, R(\delta) + Y + O], \\
				-(1 - \lambda_0)t + R(\delta) & \\
				~~+ Y + O,  & \text{if}~t \in \left[1, \frac{1}{1 - \lambda_0} (R(\delta) + Y + O) \right], \\
				0, & \text{if}~t \in \left[ \frac{1}{1 - \lambda_0} (G(\delta) + Y + O),  U(\delta) \right].
				\end{array}
				\right.
			\end{split}
			\end{equation}
			We deduce that
			\begin{equation}
				\begin{split}
					& \Phi_1(\delta) = \lim_{T \to \infty} \frac{1}{T} \int_0^T \Phi_1(t) dt = \frac{1}{U(\delta)} \int_{t = t_0}^{t_0 + U(\delta)} \Phi_1(t) dt \\
					& = \frac{\lambda_0 (R(\delta) + Y + O)^2}{2 (1 - \lambda_0) U(\delta)}  \\
					& \overset{\text{\eqref{Eqn:Rdelta} \eqref{Eqn:Udelta}}}{=} \frac{\lambda_0 (1 - \lambda_0 + \delta \lambda_{\max})^2 (Y + O)}{2 (1 - \lambda_0) (1 - \lambda_0 - \lambda_{\max}) [1 + \delta (\lambda_0 + \lambda_{\max})]}.
				\end{split}
			\end{equation}

			\begin{eqnarray}
				\Phi_1(0) = \frac{\lambda_0 (1 - \lambda_0) (Y + O)} {2 (1 - \lambda_0 - \lambda_{\max})},
			\end{eqnarray}
			which has another form if $\lambda_n = \lambda_1$ for $n \in \{ 1, ..., N \}$
			\begin{eqnarray}
				\Phi_1(0) = \frac{\lambda_0 (1 - \lambda_0) (Y + O)} {2 (1 - \lambda_0 - \lambda_1)}.  \label{Eqn:Phi1Delta0}
			\end{eqnarray}

			We also obtain that $\{ \Psi_n \}_{n \in \{ 1, ..., N \}}$ increases by $\lambda_n [  U(\delta) - R(\delta)]$ over the course of a cycle, and must have decreased by $R(\delta)$ over the preceding cycle if it is strictly positive at the end of a red period.
			\begin{eqnarray}
				&& \hspace{-0.25in} \Delta_{\Psi_n} = R(\delta) - \lambda_n U(\delta)  \nonumber \\
				&& \hspace{-0.25in} \overset{\eqref{Eqn:Rdelta} \eqref{Eqn:Udelta}}{=}  \! \frac{\lambda_{\max} (1 + \delta) (Y + O)}{1 - \lambda_0 - \lambda_{\max}} \! - \! \frac{\lambda_n [1 + \delta (\lambda_0 + \lambda_{\max})] (Y + O)}{1 - \lambda_0 - \lambda_{\max}}   \nonumber \\
				&& \hspace{-0.25in} = \frac{(Y + O) [(1 + \delta) \lambda_{\max} - \lambda_n [1 + \delta (\lambda_0 + \lambda_{\max})] }{1 - \lambda_0 - \lambda_{\max}}  \nonumber \\
				&& \hspace{-0.25in} > 0.
			\end{eqnarray}
			Thus, if $\Psi_n$ is strictly positive at the end of a red period, it must be smaller at the end of the preceding red period.

			We conclude that for any initial state $\Psi_n$ must eventually become 0 at the end of some red period, say at time~$t_1$. Then $\Psi_n$ increases at rate~$\lambda_n$ from time~$t_1$ to time $t_1 + G(\delta) + Y + O$, decreases at rate $1 - \lambda_n$ from time $t_1 + G(\delta) + Y + O$ to time $t_1 + G(\delta) + Y + O + \frac{\lambda_n}{1 - \lambda_n} \left(G(\delta) + Y + O\right) = t_1 + \frac{1}{1 - \lambda_n} \left(G(\delta) + Y + O\right) \leq t_1 +  U(\delta)$, to become zero at that point and remain zero until $t_1 +  U(\delta)$. Thus
			\begin{eqnarray}
				&& \hspace{-0.3in} \Psi_n(t_1 + t) = \nonumber \\
				&& \hspace{-0.3in} \begin{cases}
									   \lambda_n t, & \text{if}~t \in [0, G(\delta) + Y + O],    \\
									   -(1 - \lambda_n) t + G(\delta)  & \\
									   ~~~+ Y + O, & \text{if}~t \in \left[ 1, \frac{1}{1 - \lambda_n} \right] (G(\delta) + Y + O), \\
									   0, & \text{if}~t \in \left[ \frac{1}{1 - \lambda_n} (G(\delta) + Y + O),  U(\delta) \right].\\
				\end{cases}
			\end{eqnarray}

			We deduce that
			\begin{equation}
				\begin{split}
					& \Psi_n(\delta) = \lim_{T \to \infty} \frac{1}{T} \int_0^T \Psi_n(t) dt = \frac{1}{U(\delta)} \int_{t = t_1}^{t_1 +  U(\delta)} \Psi_n(t) dt  \\
					&= \frac{\lambda_n (G(\delta) + Y + O)^2}{2 (1 - \lambda_n) U(\delta)}  \\
					& \overset{\text{\eqref{Eqn:Gdelta} \eqref{Eqn:Udelta}}}{=} \frac{\lambda_n (1 + \delta \lambda_0 - \lambda_{\max})^2 (Y + O)} {2(1 - \lambda_n) (1 - \lambda_0 - \lambda_{\max}) [1 + \delta (\lambda_0 + \lambda_{\max})]}.
				\end{split}
			\end{equation}

			\begin{equation}
				\Psi_n(0) = \frac{\lambda_n (1 - \lambda_{\max})^2 (Y + O)}{2(1 - \lambda_n) (1 - \lambda_0 - \lambda_{\max})},
			\end{equation}
			which has another form if $\lambda_n = \lambda_1$ for $n \in \{ 1, ..., N \}$
			\begin{equation}
				\Psi_n(0) = \frac{\lambda_1 (1 - \lambda_1) (Y + O)}{2 (1 - \lambda_0 - \lambda_1)}.  \label{Eqn:PsinDelta0}
			\end{equation}

		\end{proof}

		We now use the propositions in \ref{Subsec:Optimality} to prove the optimality of the ``greenwave" control policy in a scenario with uniform cross traffic, i.e., $\lambda_n \equiv \lambda_1 = \lambda_{\max}$ for all $n \in \{ 1, ..., N \}$. We take one case as an example, i.e., $\lambda_0 + 2 \lambda_1 \geq 1$, $2 \lambda_0 + \lambda_1 \geq 1$, and then we obtain
		\begin{eqnarray}
			&& \Phi_1(\delta) = \frac{\lambda_0 (1 - \lambda_0 + \delta \lambda_1)^2 (Y + O)^2}{2 (1 - \lambda_0) (1 - \lambda_0 - \lambda_1)^2} ,  \nonumber \\
			&& \Psi_n(\delta) = \frac{\lambda_1 (1 + \lambda_0 \delta - \lambda_1)^2 (Y + O)^2}{2 (1 - \lambda_1) (1 - \lambda_0 - \lambda_1)^2}, \nonumber
		\end{eqnarray}
		where $n \in \{ 1, ..., N \}$. Also, $p_1 = \frac{\lambda_0}{1 - \lambda_0 - \lambda_1}$, and $q_1 = \frac{\lambda_1}{1 - \lambda_0 - \lambda_1}$.

		The lower bounds in \eqref{Eqn:Phi1Geq} and \eqref{Eqn:PsinGeq} reduce to the following
		\begin{eqnarray}
			&& \Phi_1 \geq \frac{\lambda_0 (1 - \lambda_1)^2 (Y + O)}{2 (1 - \lambda_0) (1 - \lambda_0 - \lambda_1)}, \label{Eqn:Phi1GeqA} \\
			&& \Psi_n \geq \frac{\lambda_1 (1 - \lambda_0)^2 (Y + O)}{2 (1 - \lambda_1) (1 - \lambda_0 - \lambda_1)}, \label{Eqn:PsinGeqA}
		\end{eqnarray}
		which are strict unless $G_1 = \frac{\lambda_0 (Y + O)}{1 - \lambda_0 - \lambda_1}$ and $R_n = \frac{\lambda_1 (Y + O)}{1 - \lambda_0 - \lambda_1}$, respectively.



		It is easily verified that, as $\delta \downarrow 0$, the values of $\Phi_1(\delta)$ and $\Psi_n(\delta)$, $n \in \{ 1, ..., N \}$, approach the absolute lower bounds in \eqref{Eqn:Phi1GeqA} and \eqref{Eqn:PsinGeqA}, respectively. The optimality of the ``greenwave" control policy, in the sense that no other control policy can achieve lower values for any $\Phi_n$ or $\Psi_n$ (while maintaining stability). In fact, the greenwave control policy is essentially the unique optimal policy in the scenario under consideration.

		Indeed, the optimality immediately forces $G_n = G(0) = \frac{\lambda_0 (Y + O)}{1 - \lambda_0 - \lambda_1}$, $n \in \{ 1, ..., N \}$, and $R_1 = R(0) = \frac{\lambda_1 (Y + O)}{1 - \lambda_0 - \lambda_1}$, because otherwise the inequations in \eqref{Eqn:Phi1GeqA} and \eqref{Eqn:PsinGeqA} are strict. Also, $R_n = R(0) = \frac{\lambda_1 (Y + O)}{1 - \lambda_0 - \lambda_1}$ is required for all $n = 2, \dots, N$ for stability. Now, we suppose that $R_n > R(0)$ for some $n = 2, \dots, N$, or that the non-green periods do not strictly coincide at all the nodes, and let $W_n$ be the average elapsed simultaneous non-green period at all nodes $\{ 1, \dots, n - 1 \}$ at the start of a green period at node~$n$. Then there must be at least one node~$n$ for which $W_n > 0$. Now observe that $\Phi_n$ must grow at least at rate~$\lambda_0$ when all the upstream nodes $\{ 1, \dots, n - 1 \}$ are in a green period while node $n$ is not. It may then be shown that $\Phi_n \geq \frac{\lambda_0 W_n^2}{2 (1 - \lambda_0) U_n} > 0$, preventing the optimality. Therefore, the unique optimality of the greenwave control policy is proved.

	\end{appendix}

\end{document}